\documentclass[sigconf]{acmart}

\usepackage{booktabs}
\usepackage{listings}
\usepackage{amsmath}
\usepackage{amsthm}
\usepackage{microtype}       
\usepackage[htt]{hyphenat}   
\usepackage{xurl}            
\usepackage{graphicx}
\usepackage{subcaption}
\usepackage{tikz}
\usetikzlibrary{positioning, fit, arrows.meta, backgrounds, calc}
\usepackage{tabularx}
\usepackage{rotating}      
\usepackage{ragged2e}      
\usepackage{rotfloat}
\usepackage{array}         
\newtheorem{theorem}{Theorem}

\newtheorem{proposition}{Proposition}

\newcommand{\Prb}{\mathbb{P}}

\begin{document}
\title{\texttt{proxymate}: Diagnosis and Adjustment of Proxy Estimates for Reliable Inference}






  \settopmatter{authorsperrow=4}

  \author{Alexandra N. M. Darmon}
  \email{alexdarmon@meta.com}
  \affiliation{%
    \institution{Meta}
    \country{}
  }

    \author{Deeksha Sinha}
    \authornote{Both authors contributed equally to this research.}
  \email{deekshasinha@meta.com}
  \affiliation{%
    \institution{Meta}
    \country{}
  }

  \author{Steven Wilkins-Reeves}
\authornotemark[1]
  \email{stevewr@meta.com}
  \affiliation{%
    \institution{Meta}
    \country{}
  }

  \author{Caner Gocmen}
  \email{caner@meta.com}
  \affiliation{%
    \institution{Meta}
    \country{}
  }

\authorsaddresses{}
\renewcommand{\shortauthors}{Darmon et al.}

\begin{abstract}
Proxy outcomes (such as short-term behavioral signals, model predictions, or surrogate endpoints) are frequently used in place of primary outcomes that are too slow to mature, rare, or challenging to measure directly. But valid inference on a proxy does not guarantee valid inference on the primary estimate as proxy-based estimates can be systematically biased in ways that are difficult to predict, leading to improperly calibrated confidence intervals.

We present \texttt{proxymate}, a framework and open-source Python package for proxy validation and adjustment. \texttt{proxymate} organizes into four levels: The \textit{Representativity Level} (population validity), the \textit{Unit Level} (measurement quality), the \textit{Estimate Level} (decision validity), and the \textit{Domain Level} (cross-domain transportability). Within each level, \texttt{proxymate} provides diagnostic checks, and targeted adjustment strategies that map specific failures to appropriate corrections.

At Meta, \texttt{proxymate} has been adopted by many different use cases, spanning experimentation, prevalence estimation, and monitoring use cases, all facing different proxy challenges (limited human review time, long maturation window of outcomes, low detectability) and showcasing the modularity of the framework. Across all products, \texttt{proxymate} assessed and corrected millions of proxy, primary unit comparisons. It has facilitated launches across multiple work streams including enabling quick decision making on thousands of experiments.

\end{abstract}



\maketitle

\section{Introduction} 
In many scientific and industrial settings, the primary outcome of interest is difficult to measure. Outcomes may be delayed (e.g., 60-day attribution~\cite{sarig2025}), require time consuming human labeling (e.g., identifying content that violates an online platform's policies~\cite{nguyen2020clara}), or be observable only for a small, non-representative fraction of the population. As a result, practitioners frequently substitute a more readily available \emph{proxy outcome} and make decisions based on inference computed on that proxy. 

The use of proxies has been first explored, adopting the term surrogate endpoints, in medicine~\cite{prentice1989,buyse2000}. A surrogate in this context is a more readily available and perhaps less noisy outcome used in place of the primary outcome.  When no single surrogate is sufficient, a surrogate index, a combination of weak surrogates may be constructed; more commonly used in social sciences where the presence of a single strong surrogate for the desired outcome may not exist~\cite{athey2019}.  Proxy metrics in online experimentation often serve a similar role~\cite{gupta2019}. 

Even when a proxy is strongly correlated with the primary outcome, proxy-based point estimates can be systematically biased for the primary estimand and their confidence intervals miscalibrated. Two failure modes recur in practice. First, unit-level diagnostics computed on a labeled sample (calibration, accuracy, discrimination) can look adequate while the labeled sample itself is non-representative of the target population, so aggregate estimates on the full population are biased in ways the unit-level checks do not surface. Second, standard adjustment methods (e.g., Prediction-Powered Inference (PPI)~\cite{angelopoulos2023}, regression calibration, isotonic recalibration) inherit the assumptions of the sample they are fit on: applied without first verifying representativity or cross-domain stability, they can leave the bias intact or amplify it. The gap is not the individual adjustment methods, but the absence of a diagnostic layer that determines whether and which adjustment is warranted before it is applied.

Despite extensive methodological work, a persistent practical gap remains in effectively using proxy outcomes. Existing literature focuses on either \emph{adjustment} (correcting a proxy's flaws) or \emph{validation} (diagnosing whether a proxy is adequate, though the literature in this second area is more limited~\cite{tripuraneni2024}), but rarely connects the two. Furthermore, most surrogate validation frameworks assume a controlled experimental setting with a specific treatment, making them inapplicable to generic measurement scenarios such as population-level estimation or observational monitoring. No existing framework spans unit-level quality, aggregate-level decision validity, and cross-domain transportability while also addressing the representativity of the labeled sample. 

This paper introduces \texttt{proxymate} \footnote{Open-source at \url{https://github.com/facebookresearch/proxymate}}, a framework and open-source Python package for proxy metric validation and adjustment. While \texttt{proxymate} builds on established diagnostic methods (e.g., SMD for balance~\cite{austin2011}, ECE for calibration~\cite{naeini2015,guo2017}, $I^2$ for heterogeneity~\cite{higgins2003}) and established adjustment methods (e.g., PPI~\cite{angelopoulos2023}, Platt scaling~\cite{platt1999}, regression calibration~\cite{carroll2006}), the individual checks are not our contribution. Our contributions are the framework that connects them, the open-source package, and its application at Meta:

\begin{enumerate} 
\item \textbf{A four-level validation framework.} We organize proxy validation and correction into four levels: the \textit{Representativity Level} (population validity), the \textit{Unit Level} (measurement quality), the \textit{Estimate Level} (decision validity), and the \textit{Domain Level} (cross-domain transportability). 
Each level corresponds to a distinct statistical question, data structure, and set of diagnostics and adjustments (Section~\ref{sec:framework}).
\item \textbf{A theoretical diagnostic-to-adjustment bridge.} For each level, we map theoretically specific failure modes to the adjustment whose assumptions are satisfied, formally closing the gap between diagnosing a proxy problem and delivering a corrected proxy in production.
\item \textbf{Empirical application at scale.}
Through multiple production applications of \texttt{proxymate} at Meta, we demonstrate its value in identifying proxy failures, choosing among candidate proxies and performing guided adjustment. We share practical lessons about proxy adoptions or failures in production.
\item \textbf{An open-source package}. The package exposes modular validation checks, configurable planners, custom estimators for aggregate metrics such as Average Treatment Effects, and structured validation reports that can be consumed programmatically or rendered for practitioners.
\end{enumerate} 

At Meta, \texttt{proxymate} has been used across several production applications spanning experimentation, prevalence monitoring, and ML model training (six of which have been detailed in Table~\ref{tab:proxymate-impact}). These applications collectively span four distinct primary-outcome bottlenecks (long maturation windows, restricted measurement, low detectability, low sensitivity), and exercise every component of the framework: \textit{Representativity Level} validation in three, \textit{Unit Level} in all six, \textit{Estimate Level} in five, \textit{Domain Level} in three, and \textit{Estimate Level} adjustment in one. In four applications, a validated (or adjusted-then-validated) proxy was used, enabling ongoing experiment volume of approximately 7{,}500 experiments annually across the four (dominated by $\sim$7{,}000 Long Attribution Window (LAW) conversion experiments and $\sim$300 Nonpayment experiments, with the remaining two applications accounting for the balance). In the remaining two, all candidate proxies were rejected, preventing use of a personalization model trained on biased outcomes and redirecting proxy development toward an ensemble approach. Across all products where \texttt{proxymate} was applied, it assessed and corrected millions of proxy--primary unit comparisons. Sections~\ref{sec:ipl-case-study} and~\ref{sec:whatsapp-case-study} present two case studies in detail: a long-term outcome proxy validated across 22 experiments, and a scam classifier evaluated against time consuming human labels for prevalence estimation.

\subsection{Related Work} \label{sec:related_work} 

\paragraph{Surrogate endpoints and proxy inference.} The surrogate endpoint problem originates in biostatistics~\cite{prentice1989,buyse1998,buyse2000} and was extended to the social sciences via Athey et al.'s surrogate index~\cite{athey2019} and Tripuraneni et al.'s proxy selection criteria~\cite{tripuraneni2024}. \texttt{proxymate} complements selection by validating a given proxy across multiple levels of aggregation.

{\sloppy \paragraph{Prediction-powered inference and post-prediction adjustment.} Prediction-Powered Inference (PPI)~\cite{angelopoulos2023,angelopoulos2023ppipp} debiases model-based estimates with valid CIs; Wang et al.~\cite{wang2020} and Ji et al.~\cite{ji2025} broaden this into post-prediction inference; Wilkins-Reeves et al.~\cite{wilkinsreeves2026} model residual proxy bias as a cross-domain random effect. These are adjustment mechanisms; \texttt{proxymate} provides the diagnostic layer that determines \emph{whether} and \emph{which} adjustment to apply, and we prove that PPI applied to an already-adequate proxy strictly increases MSE (Theorem~\ref{thm:unnecessary}). 
We also draw on established diagnostics from the calibration~\cite{platt1999,guo2017,naeini2015}, covariate shift~\cite{shimodaira2000,austin2011}, and meta-analysis~\cite{dersimonian1986,higgins2003} literature. \par}

{\sloppy \paragraph{Existing software.} Causal-inference packages (DoWhy~\cite{sharma2020}, EconML) focus on treatment effect estimation; \texttt{ppi\_py} implements PPI without diagnostic checks; experimentation platforms (GrowthBook, Statsig) lack formal proxy validation. \texttt{proxymate} fills this gap with a unified diagnose-then-correct workflow across four levels. \par}

We present a more comprehensive discussion of the related fields of literature in Appendix \ref{sec:extended_related_work}. 

\section{Problem Setup and Definitions}
\label{sec:pb-setup}

We introduce our data setup. Let $Y_i$ denote the primary outcome of unit $i$, let $Y^*_i$ denote the unit's proxy outcome, and let $X_i$ denote any auxiliary covariates. We also let an optional covariate $K_i \in \{1, \dots, K\}$ denote the domain index of unit $i$. We use the term \emph{domain} to denote a distinct data-generating process, arising from structural shifts such as temporal drift, geographic expansion, or distinct interventions. This is a deliberate distinction from \emph{segments}, which are subpopulations defined by conditioning on covariates within a single domain (which may be indicated through the auxiliary covariates $X_i$).

We assume that within each domain $k$ (i.e.\ for units with $K_i = k$) the unit-level data are independently distributed such that $O_i = (Y_i, Y^*_i, X_i) \sim P_k$, where $P_k$ is the population distribution within domain $k$. Typically, domains $k = 1, \dots, K-1$ form a labelled \emph{historical} dataset where $Y_i$ is observed, whereas domain $K$ is an unlabelled \emph{target} domain where $Y_i$ is not available. We are interested in estimating a parameter $\theta_K$, a functional of the target population distribution $P_K$, $\theta_K = \Psi(P_K)$ (for instance $\theta_K = \mathbb{E}_{P_K}[g(Y)]$ for some unit-level function $g$, such as $g(Y)=Y$ for a mean or $g(Y)=\mathbf{1}\{Y=1\}$ for a prevalence). The corresponding proxy-side quantity is $\theta^*_k = \Psi^*(P^*_k)$ where $\Psi^*$ is a functional of the proxy population distribution; in many settings $\Psi^* = \Psi$ though this is not required. 

\emph{Notational convention (probability vs.\ distribution).} Uppercase $P$ with a subscript (e.g., $P_k$) denotes a probability \emph{distribution} (population measure); the corresponding lowercase $p_k(x)$ denote its density (or PMF). We reserve $\mathbb{P}_k(\cdot)$ for the probability of an \emph{event} under the measure $P_k$, e.g., $\mathbb{P}_k(Y = 1 \mid X = x)$ or $\mathbb{P}_k(\theta \in \text{CI})$. Expectations under a distribution $P_k$ are written $\mathbb{E}_{P_k}[\cdot]$. We drop the subscript $k$ when there is no ambiguity about the domain.

The core challenge is that $\hat{\theta}^*_K$ is generally a biased estimator of $\theta_K$, from two sources: (i) \emph{measurement error}, where $Y^*$ is an imperfect representation of $Y$ at the unit level; and (ii) \emph{distribution shift}, where $P_K$ differs from the historical $P_{1}, \dots, P_{K-1}$, so the proxy-primary relationship observed historically may not hold in the target domain.


\subsection{Validation data}
Proxy validation relies on the historical domains $k = 1, \dots, K-1$, where both the primary outcome $Y$ and the proxy outcome $Y^*$ are observed. Within each historical domain $k$, we refer to the data containing paired $(Y, Y^*, X)$ observations as the \emph{validation data} $\mathcal{V}_k \sim P_k$. In the simplest case, $\mathcal{V}_k$ is the full domain population $P_k$ itself, or a random sample from it. In general, the outcomes $Y, Y^*$ may undergo some missingness and may not be observed from all units. If there is a non-uniform missingness, a first step is required to weight the population with respect to the target population of interest. 


The representativity of $\mathcal{V}_k$ within domain $k$ is a separate concern from whether the proxy-primary relationship generalizes from historical domains to the target domain $K$. The former is a within-domain population validity issue addressed by the \textit{Representativity Level}; the latter is a cross-domain transportability question addressed by the \textit{Domain Level}.

\section{The \texttt{proxymate} Framework}
\label{sec:framework}

\texttt{Proxymate} decomposes proxy diagnosis and adjustment into four levels. The first (Representativity) level addresses \emph{population validity}: whether the validation data on which all checks are computed should be reweighted before use. The remaining three levels address \emph{proxy quality} at increasing levels of aggregation. \textit{Unit Level} diagnoses whether the proxy tracks the primary at the \emph{individual} level. \textit{Estimate Level} diagnoses whether the proxy estimate gives the same \emph{aggregate answer} as the primary estimate within a single domain. \textit{Domain Level} diagnoses whether the proxy-primary relationship is \emph{stable across domains}.

Each level pairs a diagnostic question with the family of adjustments whose assumptions the diagnostics test. Further, each level plays a distinct role: collapsing \textit{Unit} and \textit{Estimate} would obscure that individual errors can systematically compound in aggregation, while collapsing \textit{Estimate} and \textit{Domain} would conflate within-domain bias with cross-domain instability. 

\begin{figure}[t]
\centering
\includegraphics[width=\columnwidth]{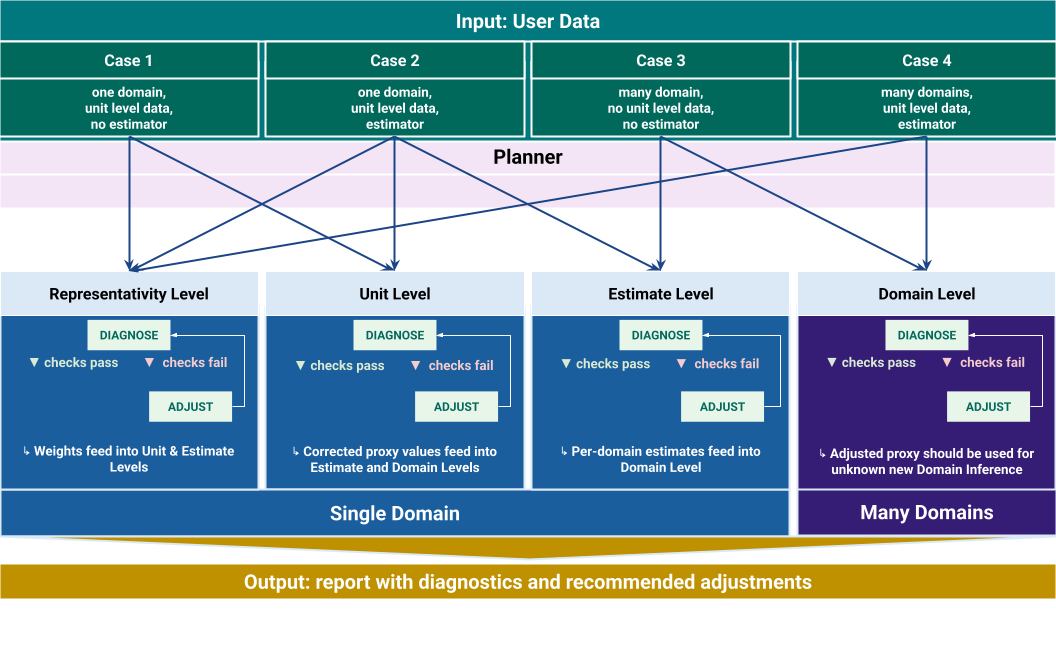}
\caption{The \texttt{proxymate} framework overview.}
\label{fig:framework}
\end{figure}


\paragraph{Segment-level validation.}
A proxy that passes on the overall population may fail for specific segments (subpopulations defined by conditioning on covariates within a single domain). 
\texttt{Proxymate} supports running all \textit{\textit{Unit}} and \textit{Estimate Level} checks per segment, and segment-level results can be aggregated into a \textit{Domain Level} meta-analysis to assess stability across subpopulations.

\paragraph{Diagnosis and Correction.} Blind adjustment can inflate MSE (Theorem~\ref{thm:unnecessary} in Appendix~\ref{app:bridge-proofs}) or amplify bias under model misspecification; the diagnostic layer selects an adjustment only when its assumptions hold. Table~\ref{tab:bridge} summarises the diagnostic-to-adjustment map.



\subsection{\textit{Representativity Level}: Population Validity}
\label{sec:representativity}
\paragraph{Should we weight the validation data, and if so, which weights?} Proxy validation in each domain $k$ relies on validation data $\mathcal{V}_k$ containing paired observations $(Y_i, Y^*_i, X_i)$. As defined in Section~\ref{sec:pb-setup}, $\mathcal{V}_k$ may be the full domain population $P_k$ itself, a random sample from it, or a separate process. The key requirement is that after the outcomes $(Y_i,Y_i^*)$ are representative of the target population $P_k$ after weighting. If no correction is needed the \textit{Representativity Level} is skipped, but when it is violated (e.g., a review queue over-representing flagged content) all downstream diagnostics are computed on the wrong distribution and may not reflect true proxy quality. IW reweighting is a sample-level correction: it makes any estimate computed on $\mathcal{V}_k$, primary or proxy, target the population $P_k$ rather than the biased sub-sample.


\paragraph{Importance-weighted estimation.} When the outcomes are not representative of $P_k$, importance weighting provides a principled correction. One approach leverages a \emph{covariate shift assumption}, involves $(Y_i, Y_i^*)|X_i = x$ regardless of whether $(Y_i,Y_i^*)$ are observed. This allows for correction based on estimating the importance weights conditional on $x$. Let $R_i = 1$ denote that the primary outcome $Y_i$ is observed for unit $i$ (i.e., unit $i$ belongs to the validation set $\mathcal{V}_k$), and $R_i = 0$ otherwise (the primary is missing, though the proxy $Y^*_i$ may still be available, though this may also be generalized to where $Y_i^*$ also has it's own missingness indicator $R_i^*$). Then $\theta_k = \mathbb{E}_{P_k}[g(Y)]$ can be recovered from $\mathcal{V}_k$ by reweighting~\cite{horvitz1952,shimodaira2000}:
$ \mathbb{E}_{P_k}[g(Y)] = \mathbb{E}_{P_k}\!\left[w(X) \cdot g(Y)|R= 1\right]$ where $w(X) = \frac{p_k(X)}{p_k(X|R = 1)}$ are the importance weights. Throughout, the unit-level function $g$ (identity for a mean, indicator for a prevalence, etc.) is used consistently in both the body and appendix.
The resulting weighted estimator is unbiased when the weights are correctly specified, but its variance increases with the variability of the weights.

\paragraph{Known vs.\ modeled weights.} The reliability of the correction depends on how the weights are obtained. \emph{Known weights} arise from design-based sampling, where inclusion probabilities are set by the practitioner (e.g., stratified sampling with known rates, or a review queue with known selection criteria). These weights are correct by construction. \emph{Modeled weights} arise when the selection mechanism is unknown and must be estimated from observed covariates, typically via a propensity model~\cite{rosenbaum1983}. Modeled weights can exhibit high variance when the validation sample and domain population diverge substantially, a phenomenon quantified by the effective sample size (ESS)~\cite{kish1965}. 

\subsubsection{Diagnostics.} \texttt{proxymate} provides five representativity checks (Table~\ref{tab:repr_checks} in Appendix \ref{sec:app_complete_checks}). \textbf{Sample representativity:} Is $\mathcal{V}_k$ drawn from the same covariate distribution as $P_k$? This is measured by standardized mean differences (SMD)~\cite{austin2011,normand2001} and the two-sample Kolmogorov-Smirnov (KS) statistic per covariate; an SMD below 0.1 is considered negligible imbalance~\cite{austin2011}. \textbf{Effective sample size:} Are the importance weights stable? This is measured by Kish's ESS~\cite{kish1965}. \textbf{Weight extremity:} Are any weights dominant? This is measured by the max/min ratio and coefficient of variation. \textbf{Weighting effectiveness:} Does importance weighting (IW) actually reduce the covariate gap? This is measured by the reduction in SMD after reweighting. \textbf{Coverage:} What fraction of units in $P_k$ are represented in $\mathcal{V}_k$? This is measured by $|\mathcal{V}_k|/|P_k|$.

\subsubsection{Adjustments.} The \textit{Representativity Level} produces one of three outcomes: 
\begin{itemize}
    \item \textbf{No Gap:} the data is representative and no weighting is needed.
    \item \textbf{IW Adjustment:} a gap is detected and importance weights are applied to all downstream checks. All \textit{\textit{Unit}} \textit{Level} checks, \textit{Estimate Level} aggregations, and per-domain estimates entering the \textit{Domain Level} meta-analysis inherit the corrected distribution.
    \item \textbf{Gap is too severe:} the gap is too severe for weighting to correct, and the framework flags the validation as unreliable for the full domain population, recommending fallback strategies (targeted label acquisition, narrowing the target population, sensitivity analysis, or segment-restricted validation described in Appendix~\ref{app:weighting_fails}). Case Study~B (Section \ref{sec:whatsapp-case-study}) illustrates a scenario where the \textit{Representativity Level} detects a severe gap that IW cannot fully correct.
\end{itemize}

\subsection{\textit{Unit Level}: Measurement Quality}

\paragraph{Does the proxy track the primary at the individual level?} The \textit{Unit Level} evaluates the unit-level discrepancy between $Y_i$ and $Y^*_i$ using the validation data $\mathcal{V}_k$ within a historical domain $k$. 

\subsubsection{Diagnostics.}

The \textit{Unit Level} evaluates six complementary properties of unit-level proxy quality. \textbf{Accuracy:} Does the proxy have a systematic bias? This is tested via Welch's t-test \cite{welch1947} on $\mathbb{E}[Y^* - Y]$, with optional importance weights.
\textbf{Precision:} How noisy is the proxy? This is measured by $R^2$, RMSE, and Pearson correlation \cite{pearson1895}.
\textbf{Agreement:} Do values agree on the identity line? This is measured using Bland-Altman \cite{bland1986}, directional concordance and Lin's Concordance Correlation Coefficient (CCC) \cite{lin1989} which decomposes as $\text{CCC} = \rho \cdot C_b$, separating correlation from location/scale mismatch.
\textbf{Calibration:} Do predicted probabilities match observed frequencies (essential for prevalence estimation)? This is measured by Expected Calibration Error (ECE) \cite{naeini2015} and the Brier score \cite{brier1950}.
\textbf{Discrimination:} Does the proxy correctly rank units? This is measured by Kendall's $\tau$ \cite{kendall1938} for continuous outcomes and Receiver Operating Characteristic - Area Under the Curve (ROC AUC) \cite{hanley1982} for binary. Note that this is independent of calibration.
\textbf{Distributional fidelity:} Do the marginal distributions match? This is measured by KS test \cite{kolmogorov1933}, Total Variation (TV) distance, and Wasserstein-1 \cite{villani2009}. A complete list of the \textit{Unit Level} checks can be found in Table~\ref{tab:unit_checks} in Appendix \ref{sec:app_complete_checks}.

\subsubsection{Adjustments.}

When \textit{\textit{Unit}} \textit{Level} diagnostics reveal failures, \texttt{proxymate} maps each failure mode to a targeted adjustment:

\begin{itemize}
\item \textbf{Systematic bias} (accuracy failure): Additive bias adjustment subtracts the estimated mean difference $\mathbb{E}[Y^* - Y]$ from proxy scores.
\item \textbf{Miscalibration} (calibration failure): Platt scaling \cite{platt1999} fits a logistic regression $\Prb(Y=1 \mid Y^*) = s(aY^* + b)$, where $s(z)=1/(1+e^{-z})$ is the logistic sigmoid, appropriate when miscalibration is approximately monotone. Isotonic regression \cite{zadrozny2002} fits a non-parametric monotone function and is more flexible but requires a larger labeled sample. For proxies that must be calibrated across subgroups (e.g., different content types or user segments), multicalibration methods such as MCGrad \cite{tax2026mcgrad} provide scalable post-hoc calibration that ensures calibration holds within data-driven subgroups, not just globally.
\item \textbf{Poor discrimination}: This cannot be fixed by post-hoc adjustment and requires proxy replacement or model retraining.
\end{itemize}


\subsection{\textit{Estimate Level}: Decision Validity}

\paragraph{Does the proxy lead to the same aggregate conclusion as the primary within a single domain?} The \textit{Estimate Level} assesses whether the proxy is reliable for aggregate inference (point estimates and confidence intervals) within a single domain. Even if a proxy is well-calibrated at the unit level, individual errors can compound or cancel in aggregation in unpredictable ways. A proxy that is noisy but unbiased at the unit level may produce reliable aggregate estimates; a proxy that is accurate on average may still produce biased aggregates if its errors correlate with treatment assignment or covariates. When the \textit{Representativity Level} has identified a gap, all aggregate estimates at this level are computed using the importance-weighted validation data $\mathcal{V}_k$.


The \textit{Estimate Level} checks are computed over $B$ bootstrapped samples from the domain.
Let $\hat\theta_k^j$ and $\hat\theta^{*j}_k$ denote the primary and proxy aggregate estimates computed over a bootstrapped sample indexed $j$ in a given domain $k$. We suppress the index $k$ when it is evident from the context. 

\subsubsection{Diagnostics.}

\textbf{Aggregate bias:} Is the proxy aggregate close to the primary? This is measured by the relative absolute error $|\hat\theta^j - \hat\theta^{*j}| / |\hat\theta^j|$.
\textbf{Inferential validity:} Does the proxy CI contain the primary estimate? 
\textbf{Sign agreement:} Do primary and proxy have the same sign? This is performed as a binary directional check. A complete list of the \textit{Estimate Level} checks can be found in Table~\ref{tab:Estimate_checks} in Appendix \ref{sec:app_complete_checks}.

\subsubsection{Adjustments.}

\begin{itemize}
\item \textbf{Aggregate bias} (bias failure): Prediction-Powered Inference (PPI)~\cite{angelopoulos2023} corrects the proxy estimate using the validation data. Let $\hat\theta^*_{\text{all}}$ denote the proxy estimate computed on the full domain population (where only $Y^*$ is observed), and let $\hat\theta_{\mathcal{V}}$ and $\hat\theta^*_{\mathcal{V}}$ denote the primary and proxy estimates computed on the validation data $\mathcal{V}_k$ (where both $Y$ and $Y^*$ are observed). The PPI estimator is:
\[
\hat\theta^{\text{PPI}}_k = \hat\theta^*_{\text{all}} + \left(\hat\theta_{\mathcal{V}_k} - \hat\theta^*_{\mathcal{V}_k}\right)
\]

The adjustment term $(\hat\theta_{\mathcal{V}_k} - \hat\theta^*_{\mathcal{V}_k})$ estimates the systematic bias of the proxy from the validation data and removes it, while the full-population proxy term $\hat\theta^*_{\text{all}}$ provides variance reduction. Under the assumption that $\mathcal{V}_k$ is representative of $P_k$ (or has been reweighted to be via the \textit{Representativity Level}), the PPI estimator is unbiased for the primary estimand and its confidence interval achieves nominal coverage (see Angelopoulos et al.~\cite{angelopoulos2023}).

\item \textbf{CI miscalibration} (coverage failure): The PPI confidence interval formulas explicitly account for the uncertainty in the bias adjustment, yielding valid coverage even when the proxy's native CI is overconfident.

\end{itemize}

\subsection{\textit{Domain Level}: Cross-\textit{Domain} Transportability}

\emph{Does the proxy-primary relationship hold across many domains?} The \textit{Domain Level} tests stability of the proxy-primary bias $\phi_k := \theta^*_k - \theta_k$ across $k=1,\dots,K-1$ historical contexts. We model the population estimands as $\theta_{k} = \gamma_0 + \gamma_1\,\theta^*_k + \varepsilon_k$ with $\varepsilon_k \stackrel{\text{iid}}{\sim} \mathcal{N}(0, \tau^2)$~\cite{dersimonian1986,higgins2003}, fit via inverse-variance-weighted meta-regression~\cite{wilkinsreeves2026}. Naive OLS on the noisy estimates is inconsistent due to attenuation from $(\sigma_k^*)^2$).   Per-domain sampling variances can be estimated when unit-level data are available for each domain; when unavailable, \texttt{proxymate} falls back to unweighted OLS. The domain-level inverse-variance weights $w_k=1/\sigma^2_{k}$ are distinct from the unit-level IW weights $w(X)$ at the \textit{Representativity Level} and never co-occur; $\hat\tau^2$ estimation via DerSimonian-Laird or REML is in Appendix~\ref{app:bridge-domain}.

\subsubsection{Diagnostics.}

\textbf{Aggregate association:} Does a large proxy effect predict a large primary effect? This is measured by precision-weighted $R^2$, generalizing Buyse's $R^2_{\text{trial}}$ \cite{buyse2000} to arbitrary estimands.
\textbf{Heterogeneity:} Is the relationship stable across domains? This is measured by Cochran's Q \cite{cochran1954}, $I^2$ \cite{higgins2003}, and $\tau^2$ \cite{dersimonian1986}. High heterogeneity means the relationship may not transport to the target domain.
\textbf{Prediction accuracy:} Can we predict the primary in a held-out domain? This is measured by Leave-one-out (LOO) cross-validation MSE on the precision-weighted regression.
\texttt{proxymate} also reports effect sensitivity ($\gamma_1 \approx 1$?) and directional alignment ($\text{sign}(\hat\theta_Y) = \text{sign}(\hat\theta^*_{Y^*})$?). A complete list of the \textit{Domain Level} checks can be found in Table~\ref{tab:domain_checks} in Appendix \ref{sec:app_complete_checks}.

\subsubsection{Adjustments.}

\begin{itemize}
\item \textbf{Systematic cross-domain bias} (intercept $\gamma_0 \neq 0$ or slope $\gamma_1 \neq 1$): Regression calibration fits the linear mapping from proxy to primary estimates across historical domains and uses it to predict the primary estimate in the target domain. The calibrated estimate is:
\[
\hat\theta^{\text{cal}}_K = \hat\gamma_0 + \hat\gamma_1 \hat\theta^*_K
\]
Wilkins-Reeves et al.~\cite{wilkinsreeves2026} show that under the random-effects meta-analytic model, the calibrated confidence interval
achieves nominal coverage for the primary estimand $\theta_K$ as $K \to \infty$, with width determined by both the proxy's sampling error and the residual cross-domain heterogeneity.

\item \textbf{\textit{Estimate}-level adjustment} \cite{wilkinsreeves2026}: A complementary approach models the residual proxy bias directly as $\phi_k \sim \mathcal{N}(\rho, \gamma^2)$, where $\rho$ is the mean bias across historical domains and $\gamma^2$ is its between-domain variance. Note that $\gamma^2$ (variance of the \emph{bias} $\phi_k$) is a different quantity from $\tau^2$ (variance of the \emph{primary} $\theta_{k}$ around the regression line \cite{wilkinsreeves2026}); the two heterogeneity parameters live in the two adjustment formulations respectively and are not interchangeable. Both parameters are estimated by method of moments from the within-domain differences $d_k = \hat\theta^*_k - \hat\theta_k$.

This approach has two advantages over regression calibration: (i) it requires only aggregate estimates and covariances from historical domains (no individual-level data), making it applicable when only summary statistics are retained; (ii) it can be layered on top of any proxy estimator (including PPI-corrected ones), adjusting for whatever residual bias remains after the initial adjustment. For small $K$, a domain bootstrap procedure propagates uncertainty in $(\hat\rho, \hat\gamma^2)$ into the interval width. \texttt{proxymate} integrates both regression calibration and estimate-level adjustment as \textit{Domain Level} adjustments.

\item \textbf{High heterogeneity} (unstable relationship): When the relationship varies unpredictably across domains, point prediction is unreliable. The residual variance cannot be corrected using the above two proposed adjustments

\end{itemize}

In practice, regression calibration and the estimate-level adjustment can be applied sequentially: first calibrate the slope and intercept, then inflate the interval for residual heterogeneity.

\subsubsection{\textit{Domain Level} Limitations}

The meta-analytic model requires a sufficient number of historical domains for reliable heterogeneity estimation. With fewer than 5 domains, \textit{Domain Level} diagnostics are unreliable and the framework flags them accordingly. For \emph{point estimates} of $(\hat\gamma_0, \hat\gamma_1, \hat\tau^2)$, the method-of-moments (DerSimonian--Laird) estimator is typically stable with 10 or more domains. For \emph{interval coverage}, however, the plug-in variance under-represents finite-$K$ estimation uncertainty (Bridge~\ref{app:bridge-domain}), and we recommend a cluster/domain bootstrap over the $K-1$ historical domains, as in \cite{wilkinsreeves2026}, whenever $K \lesssim 30$


\begin{table*}[t]
\caption{Diagnostic-to-adjustment bridge. Each row maps a failure mode to the diagnostic that detects it, the adjustment that addresses it, and the formal guarantee.}
\label{tab:bridge}
\small
\setlength{\tabcolsep}{4pt}
\begin{tabular}{@{}p{4cm}p{1.8cm}p{3.5cm}p{3.2cm}p{2.2cm}@{}}
\toprule
\textbf{Failure Mode} & \textbf{\textit{Level}} & \textbf{Diagnosed By} & \textbf{Corrected By} & \textbf{Guarantee} \\
\midrule
Labeling gap & Representativity & SMD, ESS & IW reweighting & Bridge~\ref{app:bridge-ipw}; \cite{horvitz1952,shimodaira2000} \\
\midrule
Systematic unit bias & \textit{\textit{Unit}} & T-test, MAPE & Additive bias adjustment & Bridge~\ref{app:bridge-additive-bias} \\
Miscalibration & \textit{\textit{Unit}} & ECE, Brier & Platt / Isotonic / MCGrad & Bridge~\ref{app:bridge-calibration}; \cite{platt1999,zadrozny2002,tax2026mcgrad} \\
Poor discrimination & \textit{\textit{Unit}} & ROC AUC, Kendall's $\tau$ & Proxy replacement & None \\
\midrule
Aggregate bias & \textit{Estimate} & Rel.\ abs.\ error, bootstrap cov. & PPI & \cite{angelopoulos2023} \\
CI miscalibration & \textit{Estimate} & Bootstrap coverage rate 
& PPI intervals & \cite{angelopoulos2023} \\
No bias & \textit{Estimate} & Coverage $\approx$ nominal & No adjustment & Theorem~\ref{thm:unnecessary} \\
\midrule
Residual bias after adjustment & \textit{Domain} & LOO overlap rate & Est.-level adj.  & \cite{wilkinsreeves2026} \\
Cross-domain bias & \textit{Domain} & $R^2$, slope, LOO error & Regression calibration & \cite{wilkinsreeves2026} \\
High heterogeneity & \textit{Domain} & $I^2$, $\tau^2$ & Widen pred.\ intervals & \cite{wilkinsreeves2026} \\
Model misspecification & \textit{Domain} & LOO error, $R^2$ low & Warn; report bounds & --  \\
\\
\bottomrule
\end{tabular}
\end{table*}

\subsection{Extensions}
\noindent \texttt{proxymate} also includes a \emph{validation planner} that maps recurring proxy challenges (long maturation primary, restricted labels, low detectability, low sensitivity) to recommended check subsets.
Further, each check is accompanied by configurable default thresholds which determines the passing criteria for the proxy. These thresholds are based on established conventions (such as SMD~$<0.1$~\cite{austin2011}, $I^2<25\%$~\cite{higgins2003}) and practitioner experience (Appendix~\ref{sec:validationplanner}). A team making a high-stakes policy decision should tighten thresholds, while a team running exploratory analysis may relax them.

\section{Applying \texttt{proxymate} at Meta}
We now discuss how \texttt{proxymate} is used at Meta and the learnings from implementing \texttt{proxymate} to several applications. \footnote{Reported percentages, counts, and per-arm statistics in this section and its case studies have been perturbed to preserve confidentiality; relative magnitudes and qualitative conclusions are preserved.}

\subsection{System and Workflow}
\label{subsec:system-workflow}
\texttt{proxymate} is implemented as a modular Python library. All validation checks inherit from a generic abstract base class parameterized by result and input types. New checks can be added by subclassing this base without modifying the orchestrator. The data handler provides a uniform interface for using in-memory data (Pandas dataframes). The comprehensive list of diagnostics available in the package is presented in Tables~\ref{tab:repr_checks} - \ref{tab:domain_checks}. The validation planner described in Section~\ref{sec:framework} (details in Appendix~\ref{sec:validationplanner}) is implemented as a configurable routing table; practitioners can also select checks manually or programmatically. Artificial Intelligence based tools were used to assist in development of the package. Figure~\ref{fig:framework} illustrates the end-to-end \texttt{proxymate} workflow. Appendix~\ref{sec:extendedworkflow}  walks through it in detail.





\subsection{Application overview}

The \texttt{proxymate} framework has been used for multiple production measurement workflows at Meta.
Table~\ref{tab:proxymate-impact} summarizes \texttt{proxymate}'s
use across six production applications at Meta, grouped by measurement context. The applications span four distinct \emph{primary-outcome bottlenecks}: 
Long maturation primary, Restricted Labels, Low detectability and Low sensitivity,
and collectively exercise every component of the framework (Section~\ref{sec:framework}): \textit{Representativity Level} validation in three, \textit{\textit{Unit}} \textit{Level} in all six, \textit{Estimate Level} in five, \textit{Domain Level} in three, and \textit{Estimate Level} adjustment in one. 

Table~\ref{tab:proxymate-impact} illustrates \texttt{proxymate}'s \emph{bidirectional} role. In four applications a validated proxy was used, unlocking strategic decision making and enabling ongoing experiment volume of $\sim 7{,}000$ Long Attribution Window (LAW) conversion experiments annually. In the remaining two, all candidate proxies were rejected, preventing biased downstream use in Ads quality (a personalization model that would otherwise have been trained on biased outcomes) and redirecting proxy development toward an ensemble approach in Scams. Sections~\ref{sec:ipl-case-study} and~\ref{sec:whatsapp-case-study} walk
through one validated case (Nonpayment) and one rejected case (Scams) in detail.

%
%

\begin{table*}[t]
\centering
\caption{Impact of \texttt{proxymate} at Meta across six deployed applications.}
\Description{Pivoted table summarising applications of \texttt{proxymate} at Meta. Each column corresponds to one application, grouped
  by measurement context: long-term nonpayment, mobile app performance,
  and long-attribution-window conversions are experimentation
  applications; user account type measurement and scams on a messaging
  platform are prevalence-monitoring applications; ads quality
  measurement feeds ML model training. Rows report attributes of each
  application: the measurement context; the challenge preventing direct
  use of the primary outcome; a description of the measurement problem;
  which \texttt{proxymate} validation components were applied (unit-level,
  estimate-level, domain-level validation, and estimate-level
  adjustment); the data scale of the validation; whether the proxy was
  validated, validated after adjustment, or all candidates rejected;
  and the resulting downstream business or measurement impact.}
\label{tab:proxymate-impact}
\scriptsize
\renewcommand{\arraystretch}{1.2}
\setlength{\tabcolsep}{3pt}
\newcolumntype{Y}{>{\RaggedRight\arraybackslash}X}
\begin{tabularx}{\textwidth}{@{}
  >{\bfseries\RaggedRight\arraybackslash}p{1.7cm}
  Y Y Y Y Y Y @{}}
\toprule
\textbf{Application} &
\textbf{Nonpayment} &
\textbf{Mobile perf} &
\textbf{LAW conversions} &
\textbf{Account type} &
\textbf{Scams} &
\textbf{Ads quality} \\
\cmidrule(l){2-7}
 &
Long-term nonpayment &
Mobile app performance &
Long-attribution-window (LAW) conversions &
User account type measurement &
Scams on a messaging platform &
Ads quality measurement \\
\midrule

Measurement Context &
Experimentation &
Experimentation &
Experimentation &
Prevalence monitoring &
Prevalence monitoring &
ML model training \\
\midrule

Proxy challenge(s) &
Long maturation primary&
Low detectability, Restricted labels &
Long maturation primary&
Low sensitivity, Restricted labels &
Restricted labels &
Low sensitivity, Restricted labels \\
\midrule

Problem description &
Net-of-nonpayment revenue (primary outcome) matures only 60+ days post-treatment, blocking quicker experiment readouts. &
End-user A/B testing is not permitted on app, so the primary outcome (e.g., startup time) is unobtainable; employee A/B metrics have detectability and representativity gaps. &
LAW conversions (primary outcome) mature in 21 days versus a 7-day readout target; experiments historically ignored LAW impact, causing production regressions. &
The account type label (primary outcome) is human-labeled, unavailable on most units, and the labeled subset is non-uniform, and weighted prevalence estimates need defensible validation. &
The scam label (primary outcome) is human-labeled and unavailable on most units; we compare proxy options (one human-labeled, two LLM-based) for population-level prevalence. &
Ads quality signal (primary outcome) is unavailable at user level due to anonymization; user-level outcomes are needed for heterogeneous-effect personalization models. \\
\midrule

\multicolumn{7}{@{}l}{\textit{\texttt{proxymate} components applied}} \\
\cmidrule(r){1-7}

Representativity validation &
 &
\multicolumn{1}{c}{\checkmark} &
 &
\multicolumn{1}{c}{\checkmark} &
\multicolumn{1}{c}{\checkmark} &
 \\

Unit level validation &
\multicolumn{1}{c}{\checkmark} &
\multicolumn{1}{c}{\checkmark} &
\multicolumn{1}{c}{\checkmark} &
\multicolumn{1}{c}{\checkmark} &
\multicolumn{1}{c}{\checkmark} &
\multicolumn{1}{c}{\checkmark} \\

Estimate level validation &
\multicolumn{1}{c}{\checkmark} &
\multicolumn{1}{c}{\checkmark} &
\multicolumn{1}{c}{\checkmark} &
\multicolumn{1}{c}{\checkmark} &
\multicolumn{1}{c}{\checkmark} &
 \\

Domain level validation &
\multicolumn{1}{c}{\checkmark} &
 &
\multicolumn{1}{c}{\checkmark} &
 &
 &
\multicolumn{1}{c}{\checkmark} \\

Estimate level adjustment &
 &
\multicolumn{1}{c}{\checkmark} &
 &
 &
 &
 \\
\midrule

Data Scale &
100K--100M rows each for $\sim$100 experiments &
$\sim$500K rows each for 3 experiments &
$\sim$35K units across 70 experiments &
1--10M rows per time period across 50--100 time periods &
10--50K rows&
$\sim$10K rows consisting of (experiment, date, user segment) combinations \\
\midrule

Impact on proxy use &
Proxy validated &
Proxy adjusted and validated &
Proxy validated &
Proxy validated &
All proxy candidates rejected &
All proxy candidates rejected \\
\midrule

Downstream impact from proxy use / validation &
Validated proxy now powers all $\sim$300 annual experiments in the workstream and has led to significant organizational impact through multiple landed decisions. &
Validated proxy now drives $\sim$200 experiments annually and several launch decisions with estimated  multiple billion incremental user sessions. &
Availability of a valid proxy enables measurement of LAW conversions impact in $\sim$7{,}000 experiments annually. &
Validated weighting methodology now used for account type trend reporting to company leadership. &
Validation identified segments with varying proxy quality motivating an ensemble approach for further proxy development. &
Prevented deployment of a personalization model trained on biased proxy outcomes. \\
\bottomrule
\end{tabularx}
\end{table*}

\subsection{Case Study A: Long-term Outcome Prediction}
\label{sec:ipl-case-study}

\paragraph{Problem and setup.}
A financial risk team at Meta uses a nonpayment outcome that matures 60+ days post-treatment which is too slow for experiment decisions. An ML model was built to predict the matured outcome immediately, creating a candidate proxy for experiment decisions~\cite{sarig2025}. \texttt{proxymate} was used to validate the ML model as a proxy for experimental decision making.

\begin{figure}[t]
\centering
\begin{subfigure}{0.48\textwidth}
\centering
\includegraphics[width=\textwidth]{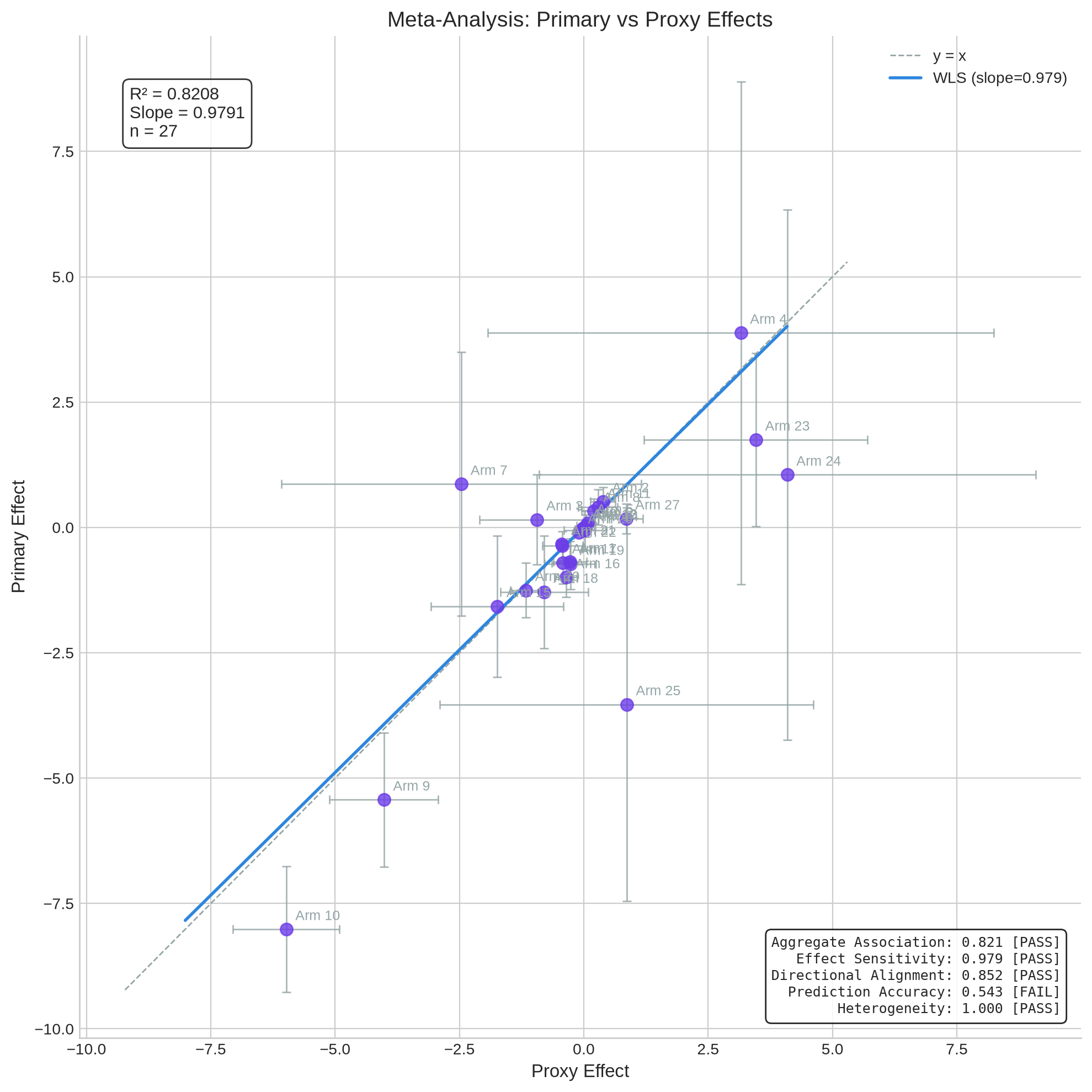}
\end{subfigure}
 \caption{Case Study A: Scatter plot capturing the \textit{Domain Level} validation across 22 experiments (27 arms) comparing the treatment effect measured through primary and proxy. }
\label{fig:casea}
\end{figure}

\paragraph{Results.} We applied \texttt{proxymate} across 22 historical experiments ($K = 27$ test versions) to validate the average treatment effect estimated through the ML model against the average treatment effect of the long maturity nonpayment outcome. Across all 27 comparisons, the proxy showed strong Unit Level agreement (directional concordance) with the primary while Unit Level precision checks ($R^2$, RMSE, Pearson $r$) failed.
This is expected for a 60-day horizon with heavy-tailed revenues, where unit-level noise is intrinsic. Evaluating at the \textit{Estimate} and \textit{Domain Levels} (§\ref{sec:framework}) then tests whether this noise compromises aggregate inference. The scatter plot between the treatment effects calculated through the primary and proxy (Figure~\ref{fig:casea}) shows strong aggregate association ($R^2 = 0.82$) and near-one effect sensitivity ($\gamma_1 = 0.98$). 
Further, $23$ of $27$ arms are directionally aligned ($85\%$). Residual heterogeneity is negligible ($I^2 \approx 0$, $\hat\tau^2 \approx 0$; Cochran's $Q = 22.6$ on $25$ degrees of freedom, with p-value $p = 0.60$). Deviations from the fit are indistinguishable from within-arm sampling noise. LOO prediction accuracy nonetheless fails (MAE $= 0.84$; median per-arm SE $= 0.41$), reflecting intrinsic noise in the $60$-day primary rather than instability in the proxy--primary relationship. Based on these results, the proxy was deemed to be safe for direction-based decision making i.e. is the treatment effect positive or not but not for measuring magnitude of the treatment effect.

\paragraph{Impact.}
The validated proxy replaced the 60-day wait with a tiered decision policy: use its direction for go/no-go, apply the estimate-level adjustment when a calibrated prediction interval is needed, and defer to the mature estimate for precise effect sizes. It now supports $\sim$300 experiments annually and has enabled multiple landed decisions.

\subsection{Case Study B: Scam Prevalence}
\label{sec:whatsapp-case-study}
\paragraph{Problem and setup.} To estimate scam prevalence on a messaging platform, expert review (primary) is available only for a small labeled sample drawn from a review queue, which over-represents flagged content and under-represents benign content, while an ML classifier (proxy) scores all content. The dataset has $>$1M target accounts; $\sim$42K have paired expert and proxy labels. We evaluate three proxy candidates: Proxy~A (scaled human review), Proxy~B and~C (two LLM classifiers). This use case exercises the \textit{Representativity}, \textit{Unit}, and \textit{Estimate Levels}.

\paragraph{Results.} The \textit{Representativity Level} flags a data-quality problem. Proxy~A passes sample representativity (SMD within bounds) but fails effective sample size: the IW weights are highly variable, indicating the labeled sample is concentrated in a small covariate region, so IW cannot fully recover population-level estimates (Appendix~\ref{app:weighting_fails}).
At the \textit{Unit Level}, Proxy~A agrees strongly with expert review on the unweighted sample. After IW reweighting to the target population, classification sensitivity (true positive rate) drops sharply: 17.3\% for Proxy~A, 31.6\% for B, 30.1\% for C. The proxies perform well on over-represented flagged accounts but poorly on the broader population. None of the proxies pass \textit{Unit Level} validation.
At the \textit{Estimate Level}, the naive expert-review prevalence is 31.9\%; the IW-corrected estimate is 3.0\% (roughly 11$\times$ lower, best-available under covariate shift, not ground truth). PPI applied to Proxy~A yields 10.7\%, down 3.3~pp from its naive proxy estimate of 13.9\%, still $\sim$3.5$\times$ the IW-corrected primary. Unlike Case A, the aggregate checks do not rescue the proxy from its \textit{Unit Level} failure, so the framework rejects all three candidates.

\paragraph{Impact.} Two corrections operate at different points in the pipeline and should not be conflated. IW is a \emph{sample-level} correction from the \textit{Representativity Level}: it reweights the biased labeled sample so any estimate computed on it (primary or proxy) targets the population rather than the labeled subset, moving the primary prevalence from 31.9\% labeled to 3.0\% population. PPI is a \emph{proxy-level} correction from the \textit{Estimate Level}: it debiases the proxy against the primary within the labeled sample (13.9\% to 10.7\% for Proxy~A) but assumes the labeled sample already represents the population. Applied without the \textit{Representativity Level}'s diagnostic, PPI would report Proxy~A at $\sim$10.7\% as the population estimate, still $\sim$3$\times$ the IW-corrected value. \texttt{proxymate}'s contribution is the ordering: representativity first flags the sample as unreliable via ESS failure; \textit{Unit Level} then rejects all three candidates after reweighting, motivating building a proxy using an ensemble approach rather than a single validated proxy.

\subsection{Learnings}
Applying \texttt{proxymate} across the applications at Meta (such as the ones in Table \ref{tab:proxymate-impact}) surfaced several recurring patterns about how proxy validation operates in practice and where it delivers value.

Validation requires paired proxy and primary signals jointly observed in the same population, but in practice the proxy is often not logged ahead of time, the primary matures only after a fresh cohort accrues, and the two frequently live in distinct data systems whose join keys must be reconstructed. Once assembled, the methodology itself runs as a single notebook execution. A second integration constraint shapes proxy selection: validated proxies must compose with the existing experimentation or reporting stack, and several applications accordingly chose a proxy that did not maximize statistical quality but fit cleanly into existing decision-rule infrastructure. This was an important production tradeoff: the best statistical proxy was not always the best deployable proxy.

In multiple applications 
the validation layer's primary contribution was to reject candidate proxies that would otherwise have entered downstream decision pipelines. Each rejection triggers a tractable next step, such as building a new proxy or changing the sampling design.

Mature applications run \texttt{proxymate} repeatedly across successive proxy metric versions 
or for ongoing monitoring. The data and sampling infrastructure assembled for the first pass should be treated as a persistent asset.

\texttt{proxymate}'s modular structure allows for easy application specific development. \textit{Estimate Level} checks can be reused to build custom sensitivity analyses (e.g., comparing a weighted estimate $\sum_{i \in N} w_i Y_i$ to a perturbed one $\sum_{i \in N^*} w_i^* Y_i$ where a subset of labels is dropped), even when the underlying weighting logic sits outside the package.

\section{Conclusion}
We presented \texttt{proxymate}, a four-level framework for proxy estimate validation that spans population validity (\textit{Representativity Level}), unit-level measurement quality (\textit{Unit Level}), aggregate-level decision validity (\textit{Estimate Level}), and cross-domain transportability (\textit{Domain Level}). Beyond assembling existing diagnostics into a pipeline, the paper contributes: (i) the four-level framework itself, applicable uniformly to A/B surrogates, prevalence classifiers, and monitoring proxies; (ii) a diagnostic-to-adjustment bridge that maps each failure mode to a principled adjustment, formalised through Theorem~\ref{thm:unnecessary}, and Bridges~\ref{app:bridge-ipw}--\ref{app:bridge-discrimination}; and (iii) empirical validation across six production applications at Meta, with two case studies (long-term outcome prediction and prevalence estimation) demonstrating that \texttt{proxymate} identifies proxy failures that would otherwise lead to incorrect decisions. Lastly (iv) a modular open-source package for practitioners.

\smallskip
\textit{Limitations.} \texttt{proxymate} validates one proxy against one primary outcome at a time and establishes statistical agreement rather than causal mechanisms. Future directions include sequential validation via confidence sequences, multi-outcome validation, and automated proxy selection using \texttt{proxymate} diagnostics as an objective.

\section{Acknowledgments}
We would like to thank Christine Agarwal, Aude Hofleitner, Thomas Leeper, Udi Weinsberg for their inputs in conceptualizing \texttt{Proxymate}. Further, we express our gratitude to Inbar Kaslasi for her help in building the \texttt{Proxymate} platform.

\bibliographystyle{ACM-Reference-Format}
\bibliography{references}

\appendix

\section{Complete Check Inventory}
\label{sec:app_complete_checks}

\noindent Tables~\ref{tab:repr_checks}--\ref{tab:domain_checks} list all 46 checks implemented in \texttt{proxymate}.

\begin{table}[t]
\caption{\textit{Representativity Level} checks (5 checks).}
\label{tab:repr_checks}
\small
\begin{tabular}{@{}p{2.8cm}p{3cm}p{2cm}@{}}
\toprule
\textbf{Check} & \textbf{Question} & \textbf{Statistic} \\
\midrule
Coverage & What fraction is covered? & $|\mathcal{V}_k|/|P_k|$ \\
Sample Representativity & Is $\mathcal{V}_k$ representative of $P_k$? & Max SMD, KS \\
Effective Sample Size & Are weights stable? & Kish's ESS \\
Weight Extremity & Any dominant weights? & Max/min, CV \\
Weighting Effectiveness & Does IW reduce gap? & Gap reduction \\
\midrule
\multicolumn{3}{@{}p{7.8cm}}{\footnotesize \emph{Abbreviations:} $\mathcal{V}_k$~=~validation (labeled) data in domain~$k$; $P_k$~=~population distribution in domain~$k$;
$|\mathcal{V}_k|$ and $|P_k|$ is number of units in the validation data and underlying population in domain~$k$ respectively; 
SMD~=~standardized mean difference between $\mathcal{V}_k$ and $P_k$ per covariate; Max SMD~=~maximum SMD across all covariates; KS~=~two-sample Kolmogorov--Smirnov statistic per covariate; ESS~=~Kish's effective sample size, $(\sum_i w_i)^2 / \sum_i w_i^2$; CV~=~coefficient of variation of the weights; IW~=~importance weighting; Gap reduction statistic is the reduction in Max SMD after reweighting.} \\
\bottomrule
\end{tabular}
\end{table}

\begin{table}[t]
\caption{\textit{Unit Level} checks (27 checks).}
\label{tab:unit_checks}
\small
\begin{tabular}{@{}p{2.5cm}p{3.2cm}p{2cm}@{}}
\toprule
\textbf{Check} & \textbf{Question} & \textbf{Statistic} \\
\midrule
\multicolumn{3}{@{}l}{\emph{Accuracy}} \\
T-Test              & Mean diff.\ = 0?         & $t$-statistic \\
MAPE                & Avg.\ \% error?          & $|Y^*-Y|/|Y|$ \\
\midrule
\multicolumn{3}{@{}l}{\emph{Precision}} \\
$R^2$               & Variance explained?      & $R^2$ \\
RMSE                & Typical error?           & $\sqrt{\text{MSE}}$ \\
Pearson $r$         & Linear association?      & $r$ \\
\midrule
\multicolumn{3}{@{}l}{\emph{Agreement}} \\
Lin's CCC           & Identity line?           & CCC \\
Bland-Altman        & Limits of agreement?     & bias $\pm$ LoA \\
Dir.\ Concordance   & Sign agreement?          & \% concordant \\
\midrule
\multicolumn{3}{@{}l}{\emph{Calibration}} \\
ECE                 & Well-calibrated?         & ECE \\
Brier Score         & MSE of probs?            & Brier \\
Brier Skill         & vs.\ base rate?          & BSS \\
Calib.\ Ratio       & Pred/obs mean?           & $\bar Y^*/\bar Y$ \\
Kuiper/ECCE         & Calibration deviation?   & ECCE \\
\midrule
\multicolumn{3}{@{}l}{\emph{Discrimination}} \\
ROC AUC             & Rank probability?        & AUC \\
PR AUC              & Prec-recall AUC?         & PR-AUC \\
Kendall $\tau$      & Rank correlation?        & $\tau$ \\
Spearman $\rho$     & Monotonic assoc.?        & $\rho$ \\
Classif.\ Accuracy  & Overall accuracy?        & $(TP{+}TN)/N$ \\
Classif.\ Precision & Precision?              & $TP/(TP{+}FP)$ \\
Sensitivity         & Recall / TPR?           & TPR \\
Specificity         & TNR?                    & TNR \\
NPV                 & Neg.\ pred.\ value?     & $TN/(TN{+}FN)$ \\
MCC                 & Matthews corr.?         & MCC \\
Confusion Matrix    & Full confusion?         & $2{\times}2$ table \\
\midrule
\multicolumn{3}{@{}l}{\emph{Distributional}} \\
KS Distance         & Max CDF diff?           & $\sup|F{-}G|$ \\
TV Distance         & Total variation?        & TV \\
Wasserstein-1       & Earth mover's?          & $W_1$ \\
\midrule
\multicolumn{3}{@{}p{7.9cm}}{\footnotesize \emph{Abbreviations:} MAPE~=~mean absolute percentage error; RMSE~=~root MSE; CCC~=~concordance correlation coefficient; LoA~=~limits of agreement; ECE~=~expected calibration error; BSS~=~Brier skill score; ECCE~=~expected calibration classification error; AUC~=~area under curve; PR-AUC~=~precision-recall AUC; TPR~=~true positive rate; TNR~=~true negative rate; NPV~=~negative predictive value; MCC~=~Matthews correlation coefficient; KS~=~Kolmogorov--Smirnov; TV~=~total variation; $W_1$~=~Wasserstein-1 distance.} \\

\bottomrule
\end{tabular}
\end{table}

\begin{table}[t]
\caption{\textit{Estimate Level} checks (8 checks).}
\label{tab:Estimate_checks}
\small
\begin{tabular}{@{}p{2.5cm}p{3.2cm}p{2cm}@{}}
\toprule
\textbf{Check} & \textbf{Question} & \textbf{Statistic} \\
\midrule
Aggregate Bias & Proxy close to primary? & Rel.\ abs.\ error \\
Inferential Val. & Proxy CI covers primary? & Bootstrap cov. \\
CI Coverage & Coverage across resamples? & Empirical cov. \\
CI Overlap & CIs overlap? & Overlap coeff. \\
Relative Precision & Proxy as precise? & SE ratio \\
Sign Agreement & Same direction? & Binary \\
Effect Size Agr. & Within each other's CI? & Bidir.\ coverage \\
Unbiasedness & Bootstrap mean centered? & Bias/SE \\
\midrule
\multicolumn{3}{@{}p{7.9cm}}{\footnotesize
\emph{Abbreviations.} Rel.\ abs.\ error $= |\hat\theta_k - \hat\theta^*_k|/|\hat\theta_k|$;
CI $=$ confidence interval; Bidir.\ coverage $=$ both CIs contain each other's
point estimate; SE $=$ standard error; Overlap coeff.\ $=$ Jaccard-style overlap
of the primary and proxy CIs.
\emph{Bootstrap.} $B$ nonparametric resamples of the labeled rows, given percentile (e.g., $95\%$) CI. \emph{Bootstrap cov.}\ $=$ indicator that the resampled $\hat\theta^*_k$ CI contains $\hat\theta_k$; \emph{Empirical cov.}\ $=$ fraction of $B$ resamples whose recomputed CI contains $\hat\theta_k$;
\emph{Bias/SE} $= |\bar\theta^*_{\text{boot}} - \hat\theta^*_k|/\widehat{\mathrm{SE}}(\hat\theta^*_k)$.} \\
\bottomrule
\end{tabular}
\end{table}

\begin{table}[t]
\caption{\textit{Domain Level} checks (6 checks).}
\label{tab:domain_checks}
\small
\begin{tabular}{@{}p{2.5cm}p{3.2cm}p{2cm}@{}}
\toprule
\textbf{Check} & \textbf{Question} & \textbf{Statistic} \\
\midrule
Aggregate Assoc. & Large proxy $\to$ large primary? & $R^2$ \\
Effect Sensitivity & One-to-one? & Slope $\gamma_1$ \\
Dir.\ Alignment & Same direction? & Concordance \\
Prediction Acc. & Predict out-of-sample? & LOO MSE/MAE \\
Heterogeneity & Stable across domains? & $I^2$, $\tau^2$, Q \\
Residual Autocorr. & Independent residuals? & Durbin-Watson \\
\midrule
\multicolumn{3}{@{}p{7.9cm}}{\footnotesize \emph{Abbreviations:} Aggregate Assoc.~=~aggregate association; Dir.\ Alignment~=~directional alignment; Prediction Acc.~=~prediction accuracy; LOO~=~leave-one-out cross-validation; MSE~=~mean squared error; MAE~=~mean absolute error; $R^2$~=~coefficient of determination (proportion of variance in the primary explained by the proxy across domains); $\gamma_1$~=~slope of the primary-on-proxy meta-regression; $I^2$~=~proportion of total variance attributable to between-domain heterogeneity; $\tau^2$~=~between-domain variance in the units of the effect; Q~=~Cochran's $Q$ statistic; Durbin-Watson~=~Durbin-Watson autocorrelation statistic.} \\
\bottomrule
\end{tabular}
\end{table}

\section{Assumptions and Guarantees for Supported Adjustments}
\label{app:bridge-proofs}

This appendix formalizes the ``diagnostic-to-adjustment bridge'' summarized in Table \ref{tab:bridge}.
Each bridge result has two parts: (i) an assumption that the diagnostic is intended to
check (or approximate), and (ii) a guarantee that the corresponding adjustment is valid
when the assumption holds.

Throughout this appendix, $P$ corresponds to the domain population $P_k$ of the main
body and $Q$ corresponds to the conditional distribution $P_k(\cdot \mid R_i = 1)$, i.e., the law of observations for units whose primary outcome $Y_i$ is observed. Expectations
under $P$ and $Q$ are denoted $\mathbb{E}_P[\cdot]$ and $\mathbb{E}_Q[\cdot]$, and $\mathbb{P}(\cdot), \mathbb{Q}(\cdot)$ denotes the probability of an event (Section~\ref{sec:pb-setup} convention).

We consider generic estimands of the form
\[
\theta := \mathbb{E}_{P}[g(Y)]
\qquad\text{and}\qquad
\theta^* := \mathbb{E}_{P}[g(Y^*)],
\]
where $g$ is a measurable function (e.g., identity for means, indicator for prevalence,
or any per-unit contribution to an aggregate estimand).
When required, we assume integrability so these expectations exist.

\subsection{Bridge 1: Representativity diagnostics $\rightarrow$ IW Reweighting}
\label{app:bridge-ipw}

\paragraph{Assumption (Selection on observables / covariate shift).}
The labeled sample $Q$ satisfies:
\begin{enumerate}
\item Conditional invariance:
\begin{equation}
\mathbb{P}(Y,Y^* \mid X) = \mathbb{Q}(Y,Y^* \mid X).
\label{eq:covshift}
\end{equation}
\item Positivity: if $p(x)>0$ then $q(x)>0$.
\end{enumerate}
Define the importance weight
\[
w(x) := \frac{p(x)}{q(x)}.
\]

\begin{proposition}[IW identity]
\label{prop:ipw-identity}
Under \eqref{eq:covshift} and positivity, for any integrable function $g(O)$,
\[
\mathbb{E}_{P}[g(O)] = \mathbb{E}_{Q}\big[w(X)\, g(O)\big].
\]
\end{proposition}

\begin{proof}
By the law of iterated expectation and \eqref{eq:covshift},
\begin{align*}
\mathbb{E}_Q[w(X)g(O)]
&= \int w(x)\, \mathbb{E}_Q[g(O)\mid X=x]\, q(x)\, dx \\
&= \int \mathbb{E}_P[g(O)\mid X=x]\, p(x)\, dx \\
&= \mathbb{E}_P[g(O)].
\end{align*}
\end{proof}

\paragraph{Interpretation (why SMD/ESS are part of the bridge).}
Representativity diagnostics such as standardized mean differences (SMD) and
two-sample tests on $X$ are aimed at assessing whether the covariate density $q(x)$ plausibly matches $p(x)$,
and whether reweighting to align $q(x)$ to $p(x)$ is feasible. Weight stability diagnostics
(e.g., Kish effective sample size) assess whether the variance inflation induced by $w(X)$
is acceptable; when the effective sample size is small, IW estimators have high variance
and are sensitive to model misspecification and finite-sample noise, motivating caution
or fallback strategies (targeted labeling, narrowing target, sensitivity analysis).

\subsection{Bridge 2: Systematic Unit Bias Diagnostics $\rightarrow$ Additive Bias Correction}
\label{app:bridge-additive-bias}

\paragraph{Assumption (additive proxy bias).}
Suppose the proxy admits a constant additive bias
\begin{equation}
\mathbb{E}_P[Y^* - Y] = b,
\label{eq:additive-bias}
\end{equation}
for some finite $b$ (possibly domain-specific). The stronger unit-level decomposition $Y^* = Y + b + \varepsilon$ with $\mathbb{E}[\varepsilon]=0$ is sufficient but not necessary for the mean-estimand correction below.

\begin{proposition}[Additive bias correction for mean estimands]
\label{prop:additive-bias}
Assume \eqref{eq:additive-bias}. Let $b = \mathbb{E}_P[Y^*-Y]$ and define the corrected proxy
$\tilde Y^* := Y^* - b$. Then
\[
\mathbb{E}_P[\tilde Y^*] = \mathbb{E}_P[Y].
\]
If $b$ is estimated from representative paired data (or IW-reweighted paired data),
then the plug-in correction $\tilde Y^* := Y^* - \hat b$ yields an asymptotically
unbiased estimator of $\mathbb{E}_P[Y]$.
\end{proposition}

\begin{proof}
By linearity of expectation,
\[
\mathbb{E}_P[\tilde Y^*] = \mathbb{E}_P[Y^* - b] = \mathbb{E}_P[Y^*] - \mathbb{E}_P[Y^*-Y] = \mathbb{E}_P[Y].
\]
Consistency/unbiasedness of $\hat b$ follows from standard laws of large numbers under
representative sampling (or Proposition~\ref{prop:ipw-identity} under IW).
\end{proof}

\paragraph{Bridge logic.}
\textit{Unit}-level bias diagnostics (e.g., $t$-tests of $\mathbb{E}[Y^*-Y]=0$) are designed to detect
violations of $b\approx 0$. When violated, Proposition~\ref{prop:additive-bias} shows
the additive correction targets exactly this failure mode.

\subsection{Bridge 3: Miscalibration Diagnostics $\rightarrow$ Calibration Adjustments}
\label{app:bridge-calibration}

Let $Y\in\{0,1\}$ and let $S$ be a proxy score (often $S=Y^*$) intended to approximate
$\Prb(Y=1\mid\cdot)$. Define the \emph{calibration function}
\[
m(s) := \Prb(Y=1\mid S=s).
\]

\begin{proposition}[Perfect calibration via the true calibration map]
\label{prop:true-calibration-map}
The transformed score $\tilde S := m(S)$ is calibrated:
\[
\mathbb{E}[Y \mid \tilde S] = \tilde S \quad \text{a.s.}
\]
\end{proposition}

\begin{proof}
Using iterated expectations,
\[
\mathbb{E}[Y\mid m(S)] = \mathbb{E}\big[\mathbb{E}[Y\mid S]\mid m(S)\big]
= \mathbb{E}[m(S)\mid m(S)] = m(S).
\]
\end{proof}

\paragraph{Model-based calibration (Platt scaling).}
If $m(s)=\sigma(as+b)$ for some $a,b$ and logistic sigmoid $\sigma(\cdot)$, then fitting
$(\hat a,\hat b)$ by maximum likelihood yields $\hat m(s)=\sigma(\hat a s+\hat b)$ which
is consistent for $m(s)$ under standard regularity conditions; combined with
Proposition~\ref{prop:true-calibration-map}, this implies asymptotically vanishing
calibration error.

\paragraph{Nonparametric calibration (isotonic regression).}
If $m(s)$ is monotone nondecreasing, isotonic regression estimates $\hat m$ that is
$L_2$-consistent for $m$ under mild conditions, implying calibration error converges to
zero as sample size grows.

\paragraph{Bridge logic.}
Calibration diagnostics (ECE, Brier score) detect $m(s)\neq s$ in aggregate. The
corresponding adjustments (Platt/isotonic/multicalibration) aim to estimate $m$ within
an appropriate function class; Proposition~\ref{prop:true-calibration-map} provides the
target guarantee: replacing $S$ by $\hat m(S)$ yields calibrated probabilities in the
limit when $\hat m\to m$.

\subsection{Bridge 4: \textit{Estimate Level} Bias/Coverage Diagnostics $\rightarrow$ PPI}
\label{app:bridge-ppi}

Assume the proxy is available for all $N$ target units, and we have a labeled set
$L$ of size $n$ with paired $(Y_i,Y_i^*)$ drawn representatively from the target
population (or made representative via IW).
We write $g$ for the unit level function throughout. Define
\[
\theta := \mathbb{E}_P[g(Y)],\quad
\theta^* := \mathbb{E}_P[g(Y^*)],\quad
\Delta := g(Y)-g(Y^*).
\]

The prediction-powered inference (PPI) estimator is
\begin{equation}
\hat\theta_{\mathrm{PPI}}
:=
\frac{1}{N}\sum_{j=1}^N g(Y_j^*)
+
\frac{1}{n}\sum_{i\in L}\big(g(Y_i)-g(Y_i^*)\big)
=
\hat\theta^*_{\text{all}} + \hat\Delta_L.
\label{eq:ppi}
\end{equation}

\begin{proposition}[Unbiasedness of PPI]
\label{prop:ppi-unbiased}
If $L$ is representative of $P$ for the joint law of $(Y,Y^*)$ (or IW-corrected to be),
then $\mathbb{E}[\hat\theta_{\mathrm{PPI}}]=\theta$.
\end{proposition}

\begin{proof}
Taking expectations in \eqref{eq:ppi},
\[
\mathbb{E}[\hat\theta_{\mathrm{PPI}}]
=
\mathbb{E}[g(Y^*)] + \mathbb{E}[g(Y)-g(Y^*)]
=
\mathbb{E}[g(Y)]
=
\theta.
\]
\end{proof}

\begin{proposition}[Asymptotic variance and Wald intervals]
\label{prop:ppi-var}
If $\hat\theta^*_{\text{all}}$ and $\hat\Delta_L$ are computed on independent samples
(e.g., via sample splitting) then
\[
\operatorname{Var}(\hat\theta_{\mathrm{PPI}})=\frac{\sigma_*^2}{N}+\frac{\sigma_\Delta^2}{n},
\]
where $\sigma_*^2=\operatorname{Var}(g(Y^*))$ and $\sigma_\Delta^2=\operatorname{Var}(\Delta)$. Plug-in estimates
yield asymptotically valid Wald confidence intervals by the CLT.
\end{proposition}

\begin{proof}
By independence,
$\operatorname{Var}(\hat\theta_{\mathrm{PPI}})=\operatorname{Var}(\hat\theta^*_{\text{all}})+\operatorname{Var}(\hat\Delta_L)
=\sigma_*^2/N+\sigma_\Delta^2/n$. The CLT and Slutsky's theorem yield Wald validity.
\end{proof}

\paragraph{Bridge logic.}
\textit{Estimate}-level diagnostics flag (i) aggregate bias $\theta^*-\theta\neq 0$ and/or
(ii) CI miscalibration (e.g., poor coverage when treating proxy-based intervals as
intervals for $\theta$). Proposition~\ref{prop:ppi-unbiased} shows PPI corrects the
bias term, and Proposition~\ref{prop:ppi-var} shows the corrected interval must widen
by the additional $\sigma_\Delta^2/n$ uncertainty term.

\subsection{Bridge 5: ``No Bias'' Diagnostics $\rightarrow$ Skip Adjustment (Gating Rule)}
\label{app:bridge-skip}



Let $\hat{\theta}^*_K$ be a proxy estimator with bias
$\phi = \theta^*_K - \theta_K$ and variance $\sigma_*^2 / N$, and let
$\hat{\theta}^{\mathrm{PPI}}_K$ be the PPI-corrected estimator using a labeled
sample of size $n$ under the assumptions of 
Proposition~\ref{prop:ppi-unbiased} and Proposition~\ref{prop:ppi-var}, with adjustment
variance $\sigma_\Delta^2 / n$ where $\sigma_\Delta^2 = \mathrm{Var}(g(Y) - g(Y^*))$.

\begin{theorem}[Unnecessary adjustment inflates MSE]
\label{thm:unnecessary}
\[
\mathrm{MSE}(\hat{\theta}^{\mathrm{PPI}}_K) \;<\; \mathrm{MSE}(\hat{\theta}^*_K)
\;\iff\; |\phi| \;>\; \sigma_\Delta / \sqrt{n}.
\]
In particular, when the proxy is unbiased ($\phi = 0$), PPI strictly
increases MSE by $\sigma_\Delta^2 / n$.
\end{theorem}

\begin{proof}
$\mathrm{MSE}(\hat{\theta}^*_K) = \phi^2 + \sigma_*^2 / N$. Under the
sample-splitting assumption of Bridge~\ref{app:bridge-ppi} (the two sums in
$\hat{\theta}^{\mathrm{PPI}}$ are computed on independent samples, or
$n \ll N$ so the covariance term is negligible), $\mathrm{MSE}(\hat{\theta}^{\mathrm{PPI}}_K) = \sigma_*^2 / N + \sigma_\Delta^2 / n$
since PPI is unbiased. The difference is $\phi^2 - \sigma_\Delta^2 / n$.
\end{proof}

The break-even condition $|\phi| > \sigma_\Delta/\sqrt{n}$
has a practical reading: for small labeled samples (small $n$), the threshold
is high, so the adjustment is only worthwhile when the bias is substantial.

\paragraph{Bridge logic.} \emph{Estimate}-level diagnostics (aggregate bias,
bootstrap coverage) estimate $\phi$ and compare it to the
$\sigma_\Delta/\sqrt{n}$ threshold. When the diagnostics imply
$\phi \approx 0$ (e.g., near-nominal coverage for proxy intervals interpreted
as intervals for $\theta$, or small measured aggregate bias), PPI is skipped:
applying it would inject variance without reducing bias.

\subsection{Bridge 6: \textit{Domain Level} Diagnostics $\rightarrow$ Regression Calibration and Heterogeneity Inflation}
\label{app:bridge-domain}

Assume we observe $(\hat\theta_{k},\hat\theta_k^*)$ in historical domains $k=1,\dots,K-1$,
with (approximate) sampling variances $\sigma^2_{k}$ and $(\sigma_k^*)^2$ (matching the notation of \cite{angelopoulos2023}).
Assume a random-effects calibration model:
\begin{equation}
\theta_k = \gamma_0 + \gamma_1 \theta_k^* + \varepsilon_k,
\qquad
\varepsilon_k \stackrel{iid}{\sim} \mathcal{N}(0,\tau^2).
\label{eq:re-model}
\end{equation}
In the target domain $K$, only $\hat\theta_K^*$ is observed.

The predictive interval $I_{K,1-\alpha}^{\mathrm{cal}}$ for $\theta_K$ is stated as in
. The derivation
follows the random-effects model \eqref{eq:re-model} above: conditional on $\theta_K^*$,
$\theta_K$ has additive random effect $\varepsilon_K$ with variance $\tau^2$; the observed
proxy estimate $\hat\theta_K^*$ contributes sampling variance $(\sigma_K^*)^2$, which
scaled by $\gamma_1$ becomes $\gamma_1^2(\sigma_K^*)^2$. The estimator of $(\gamma_0, \gamma_1)$ is the inverse-variance-weighted meta-regression of \cite{dersimonian1986}: weights $1/\sigma^2_{k}$ (or $1/(\sigma^2_{k} + \hat\tau^2)$ for the random-effects variant) on the outcome side, with the regressor treated as observed with known variance $(\sigma_k^*)^2$. This estimator is consistent for $(\gamma_0,\gamma_1)$ under the joint asymptotic regime $K \to \infty$ with $\max_k (\sigma_k^*)^2 \to 0$ (attenuation from measurement error in the regressor vanishes in the limit); a plain OLS fit to $(\hat\theta_{k}, \hat\theta_k^*)$ is \emph{not} consistent and should not be used with this coverage claim. Under this regime, plug-in $(\hat\gamma_0, \hat\gamma_1, \hat\tau^2)$ with Slutsky's theorem yield the asymptotic coverage stated in \cite{wilkinsreeves2026}. \emph{Finite-$K$ caveat:} the plug-in variance omits the $O_p(1/K)$ uncertainty in $(\hat\gamma_0,\hat\gamma_1,\hat\tau^2)$, so the interval is narrower than nominal at small $K$; a cluster/domain bootstrap over the $K-1$ historical domains (as in Wilkins-Reeves et al.~\cite{wilkinsreeves2026}) is recommended when $K$ is small.

\paragraph{Bridge logic.}
\textit{Domain}-level diagnostics play distinct roles:
\begin{itemize}
\item Association/sensitivity diagnostics (e.g., $R^2$, slope $\gamma_1$ near 1, intercept near 0)
target whether a linear calibration map is plausible.
\item Heterogeneity diagnostics ($I^2$, $\tau^2$, $Q$) target whether $\tau^2$ is non-negligible.
When $\tau^2>0$, intervals that ignore $\tau^2$ are overconfident.
\item Prediction accuracy (e.g., LOO error) targets model adequacy; large LOO error suggests
misspecification, in which case one should warn and/or revert to conservative bounds rather than
apply regression calibration blindly.
\end{itemize}

\paragraph{Estimation of Between-Domain Variance}
The between-domain variance $\tau^2$ can be estimated by method of moments or by maximum likelihood. \texttt{proxymate} uses the DerSimonian-Laird moment estimator~\cite{dersimonian1986}, which sets Cochran's $Q$ equal to its expectation under the random-effects model and solves for $\tau^2$:
\[
\hat\tau^2_{\text{DL}} = \max\!\left(0,\; \frac{Q - (K-1)}{\sum_k w_k - \sum_k w_k^2 / \sum_k w_k}\right), \qquad Q = \sum_{k=1}^{K-1} w_k (\hat\theta_k - \bar\theta_w)^2
\]
where $w_k = 1/\sigma^2_{k}$ are the inverse-variance weights and $\bar\theta_w = \sum_k w_k \hat\theta_k / \sum_k w_k$ is the fixed-effects weighted mean. This estimator is closed-form and requires no iteration, but truncates at zero and does not use the full distributional information.
 
The alternative is restricted maximum likelihood (REML), which maximizes the marginal likelihood of the data under the assumption $\hat\theta_k \sim \mathcal{N}(\theta, \sigma^2_{k} + \tau^2)$:
\[
\ell_{\text{REML}}(\tau^2) = -\frac{1}{2} \sum_{k=1}^{K-1} \left[\log(\sigma^2_{k} + \tau^2) + \frac{(\hat\theta_k - \bar\theta_{w^*})^2}{\sigma^2_{k} + \tau^2}\right] - \frac{1}{2}\log\!\left(\sum_{k=1}^{K-1} w^*_k\right)
\]
where $w^*_k = 1/(\sigma^2_{k} + \tau^2)$ and $\bar\theta_{w^*}$ is the random-effects weighted mean. The final term is the restriction term that distinguishes REML from profile ML (it is $-\tfrac{1}{2}\log|X^\top\Sigma^{-1}X|$ with the design matrix $X$ a column of ones); dropping it recovers the profile ML likelihood, which is known to be downward-biased for $\tau^2$ at small $K$. REML requires iterative optimization but produces less biased estimates of $\tau^2$, particularly when $K$ is small. For $K > 20$, both estimators give similar results; for $K < 10$, REML is preferred~\cite{veroniki2016}.

\subsection{Bridge 7: Discrimination Failure $\rightarrow$ Proxy Replacement (Limits of Post-hoc Fixes)}
\label{app:bridge-discrimination}

Let $S$ be a score and $\tilde S = h(S)$ where $h$ is \emph{strictly} monotone (so that $h$ never creates ties between distinct values of $S$; weakly monotone $h$ can collapse distinct scores into ties, in which case AUC changes by an amount that depends on the tie-handling convention).
Rank-based discrimination measures depend only on pairwise ordering.

\begin{proposition}[Monotone recalibration cannot improve rank discrimination]
\label{prop:mono-no-auc-gain}
If $h$ is strictly monotone, then ROC AUC is invariant:
\[
\mathrm{AUC}(\tilde S) = \mathrm{AUC}(S).
\]
More generally, any rank correlation (Spearman $\rho$, Kendall $\tau$) between $S$ and $Y$
is invariant under strictly monotone transforms.
\end{proposition}

\begin{proof}
ROC AUC equals $\Prb(S_1 > S_0)$ for independent draws $(S_1\mid Y=1)$ and $(S_0\mid Y=0)$.
If $h$ is strictly monotone, $S_1>S_0 \Leftrightarrow h(S_1)>h(S_0)$, hence the probability
(and thus AUC) is unchanged. Rank correlations follow from the same order-preservation argument.
\end{proof}

\paragraph{Bridge logic.}
If discrimination diagnostics fail (low AUC/low rank correlation), post-hoc calibration
or additive shifts cannot fix the fundamental ranking information loss; the appropriate
action is proxy replacement (retraining, new features, different model), not post-hoc adjustment.

\subsection{When Weighting Fails to Correct Representativity Issues}
\label{app:weighting_fails}
If the ESS is too low or the weighting effectiveness check fails, importance weighting alone cannot recover valid population-level estimates. \texttt{proxymate} flags this condition and recommends one of four fallback strategies:
\begin{enumerate}
\item \emph{Targeted label acquisition:} acquire additional labels in under-represented regions of the covariate space to improve coverage and reduce weight variability.
\item \emph{Narrow the target population:} redefine the target $T$ to the proxy population $P$ (or the well-covered subset), accepting that validation results apply only to the narrower scope.
\item \emph{Sensitivity analysis:} report bounds on the estimand under varying assumptions about the unobserved population, rather than a single point estimate.
\item \emph{Segment-restricted estimates:} report validated estimates only for segments where the labeled sample is representative, with explicit scope limitations.
\end{enumerate}

\section{Extended \texttt{proxymate} Workflow}
\label{sec:extendedworkflow}
Figure~\ref{fig:framework} shows the end-to-end \texttt{proxymate} pipeline. A practitioner supplies data in one of four canonical forms (Section~\ref{sec:data_cases}); the \emph{Planner} routes the request to the applicable subset of the four levels based on what data is available (Section~\ref{sec:validationplanner}); each level runs a \emph{diagnose-then-adjust} loop that produces diagnostic verdicts and, where needed, targeted corrections (Section~\ref{sec:levels_loop}); the pipeline emits a validation report summarising all diagnostics and recommended adjustments (Section~\ref{sec:workflow_output}). Corrections at each level feed forward into subsequent levels, so a representativity gap detected on the input propagates its importance weights into all downstream checks.

\subsection{Data Input}
\label{sec:data_cases}
The four levels of \texttt{proxymate} are not always all executable: which ones apply depends on what data the practitioner supplies. Figure~\ref{fig:framework} summarises four canonical data-input cases, distinguished by (i) whether one or many domains are available and (ii) whether the practitioner provides unit-level data with a (custom) estimator or only pre-computed aggregate estimates.

\begin{enumerate}
\item \textbf{Case 1: Single domain, unit-level data, no estimator.} The practitioner supplies paired $(Y_i, Y^*_i, X_i)$ observations for a single domain without a custom estimator. The \textit{Representativity}, \textit{Unit}, and \textit{Estimate Levels} are all executable; the \textit{Domain Level} is not (a single domain provides no cross-domain variation to meta-analyze).

\item \textbf{Case 2: Single domain, unit-level data, custom estimator.} As Case~1, but the practitioner supplies a custom estimator (e.g., \texttt{AverageTreatmentEffectEstimator} for ATEs, a quantile estimator, or any subclass of \texttt{DfBasedEstimatorBase}). The same three levels are executable, with aggregation performed via the supplied estimator.

\item \textbf{Case 3: Many domains, aggregate summaries only.} The practitioner supplies per-domain aggregate estimates $\{(\hat\theta_k, \hat\theta^*_k, \hat\sigma_{k}, \hat\sigma_{Y^*,k})\}_{k=1}^{K-1}$ without unit-level data (common when historical experiments retained only summary statistics). Only the \textit{Domain Level} is executable, since \textit{Representativity}, \textit{Unit}, and \textit{Estimate Levels} all require paired unit-level $(Y_i, Y^*_i)$ observations.

\item \textbf{Case 4: Many domains, unit-level data, custom estimator.} The full case: paired unit-level data for multiple historical domains plus a custom estimator. All four levels are executable, and per-domain estimates feed forward into the \textit{Domain Level} meta-analysis. This is the input for cross-domain validation with proxy adjustment (Case Study~A, Section~\ref{sec:ipl-case-study}).
\end{enumerate}

Table~\ref{tab:data_cases} summarises which levels are executable in each case.

\begin{table}[t]
\caption{Executable validation levels by data-input case. `\checkmark' indicates a level that can be run; `--' indicates the required data is not available.}
\label{tab:data_cases}
\small
\setlength{\tabcolsep}{4pt}
\begin{tabular}{@{}c c c c || c c c c@{}}
\toprule
\textbf{Case} & \textbf{Domains} & \textbf{Unit data} & \textbf{Estimator} & \textbf{Repr.} & \textbf{Unit} & \textbf{Est.} & \textbf{Dom.} \\
\midrule
1 & single & yes & --      & \checkmark & \checkmark & -- & -- \\
2 & single & yes & custom  & \checkmark & \checkmark & \checkmark & -- \\
3 & many   & no  & --      & -- & -- & -- & \checkmark \\
4 & many   & yes & custom  & \checkmark & \checkmark & \checkmark & \checkmark \\
\bottomrule
\end{tabular}
\end{table}

\subsection{Validation Planner}
\label{sec:validationplanner}

A common barrier to proxy validation is deciding \emph{which} checks to run. The answer depends on two inputs: the \emph{proxy challenge} (why a proxy is needed) and the \emph{data level} (what paired data is available). \texttt{proxymate} identifies four recurring proxy challenges:

\begin{enumerate}
\item \emph{Long maturation primary}: the outcome takes weeks or months to mature.
\item \emph{Restricted labels}: ground truth requires time-consuming human review and the labeled sample may not be representative.
\item \emph{Low detectability}: the outcome is rare or high-variance.
\item \emph{Low sensitivity}: the metric is too blunt to detect small effects.
\end{enumerate}

Each challenge maps to a different subset of levels and checks (Table~\ref{tab:planner}). For example, the long maturation primary challenge routes primarily to \textit{Domain Level} checks (association, heterogeneity, prediction accuracy), since the key question is whether the proxy--primary relationship is stable across historical experiments. The restricted labels challenge instead emphasizes \textit{Representativity} and \textit{Unit Level} checks (calibration, discrimination), since the labeled sample is often biased by the review pipeline. Challenges can be combined (e.g., long maturation primary with low sensitivity).

The validation planner encodes this routing and, given a challenge and data level, returns the recommended set of checks grouped by validation level. Practitioners may also bypass the planner and select checks manually. The planner is a convenience, not a constraint: it lowers the barrier for non-expert users while preserving full flexibility for teams with domain-specific requirements.

\begin{table}[t]
\caption{Validation planner routing table. Each proxy challenge maps to recommended checks; challenges can be combined.}
\label{tab:planner}
\small
\setlength{\tabcolsep}{3pt}
\begin{tabular}{@{}p{2.2cm}p{5cm}@{}}
\toprule
\textbf{Proxy Challenge} & \textbf{Recommended Checks} \\
\midrule
Long maturation primary & Association ($R^2$), effect sensitivity ($\gamma_1$), directional alignment, heterogeneity ($I^2$, $\tau^2$), prediction accuracy (LOO) \\
\addlinespace
Restricted labels & Coverage, sample repr.\ (SMD), ESS, weight extremity, calibration (ECE), discrimination (AUC), aggregate bias \\
\addlinespace
Low detectability & Precision ($R^2$, RMSE), agreement (CCC), inferential validity, aggregate bias \\
\addlinespace
Low sensitivity & Discrimination (AUC, $\tau$), accuracy (t-test), directional alignment, effect sensitivity \\
\bottomrule
\end{tabular}
\end{table}

\subsection{\texttt{Proxymate} Levels: Diagnose-then-Adjust Loop}
\label{sec:levels_loop}

At each level selected by the Planner, \texttt{proxymate} runs a two-step \emph{diagnose-then-adjust} loop (Figure~\ref{fig:framework}). The \textbf{diagnose} step evaluates the level's checks and returns pass/fail verdicts. If all checks pass, no adjustment is applied and the level's output is a clean verdict. If any check fails, the \textbf{adjust} step applies a targeted correction whose assumptions match the specific failure mode (Table~\ref{tab:bridge}), and the level is then re-validated on the corrected data.

The four levels split into two groups reflecting the data they operate on.

\paragraph{Within-domain stage (\textit{Representativity}, \textit{Unit}, \textit{Estimate}).} For each historical domain $k$ with paired unit-level data $(Y_i, Y^*_i, X_i)$, the three within-domain levels run sequentially. \textit{Representativity} first checks whether the labeled sample is representative of the domain population; if a gap is detected, the resulting importance weights are applied to all downstream levels. \textit{Unit Level} evaluates measurement quality on the (possibly reweighted) paired data. \textit{Estimate Level} then assesses aggregate decision validity. Corrections feed forward: representativity weights flow into \textit{Unit} and \textit{Estimate} checks, and unit-level corrections (e.g., recalibrated proxy scores) propagate into aggregate estimation.

\paragraph{Cross-domain stage (\textit{Domain}).} Per-domain validated estimates $\{(\hat\theta_k, \hat\theta^*_k)\}_{k=1}^{K-1}$ from the within-domain stage (or provided directly as aggregate summaries under Case~3) feed into the \textit{Domain Level}, which assesses whether the proxy--primary relationship is stable enough to transport to the target domain~$K$. The same diagnose-then-adjust cycle applies: if cross-domain diagnostics pass, the proxy estimate is used directly; if they fail, regression calibration or estimate-level adjustment is applied to produce a calibrated prediction interval for $\theta_K$.

\subsection{Output}
\label{sec:workflow_output}

The pipeline produces a \emph{validation report} summarising, for each level executed: the diagnostic check results (raw statistic, projected score, pass/fail), the adjustments applied (if any), and the resulting corrected proxy estimates with confidence intervals. When the \textit{Domain Level} is executed, the report additionally includes the target-domain calibrated estimate $\hat\theta^{\text{cal}}_K$ and its predictive interval. The report is designed to be consumed programmatically (as a structured object, for pipeline integration) or as a human-readable summary (for one-off validation).

\section{Extended Related Work} \label{sec:extended_related_work}
\paragraph{Surrogate endpoints and indices.} The foundational work on proxy validation originates in biostatistics. Prentice~\cite{prentice1989} formalized the surrogate endpoint problem and proposed an operational criterion: a surrogate is valid if it fully captures the treatment effect on the primary outcome. This condition relies on strong causal assumptions rarely satisfied in practice. Buyse and Molenberghs~\cite{buyse1998,buyse2000} extended this to a meta-analytic framework, validating surrogates via $R^2_{\text{trial}}$, the proportion of treatment effect variance explained by the surrogate across multiple trials. VanderWeele~\cite{vanderweele2013} distinguished between statistical and causal surrogacy. In the social sciences, Athey et al.~\cite{athey2019} introduced the surrogate index for combining multiple weak surrogates into a strong one, focusing on short-term outcomes as proxies for longer-term effects. \texttt{proxymate} complements these works by validating a given proxy across multiple levels of aggregation and aligning the validation criteria with the downstream application. 

\paragraph{Prediction-powered inference.} Angelopoulos et al.~\cite{angelopoulos2023} introduced Prediction-Powered Inference (PPI), which uses a small labeled dataset to correct the bias of a model evaluated on a large unlabeled dataset, providing valid confidence intervals. PPI++~\cite{angelopoulos2023ppipp} extends this with power-tuning for efficiency gains. Ji et al.~\cite{ji2025} revisit the connection between predictions-as-surrogates and the classical surrogate outcome framework. Kluger et al.~\cite{kluger2025} extend PPI to settings with imputed covariates and nonuniform sampling, partially addressing representativity concerns. These methods are complementary to \texttt{proxymate}: PPI provides an adjustment mechanism (used in our adjustment module), while \texttt{proxymate} provides the diagnostic layer that determines \emph{whether} and \emph{which} adjustment to apply. 

\paragraph{Calibration and post-hoc adjustment.} Classifier calibration, ensuring predicted probabilities match empirical frequencies, is central to proxy quality at the unit level. Platt~\cite{platt1999} introduced sigmoid scaling, while Zadrozny and Elkan~\cite{zadrozny2002} proposed histogram and isotonic calibration. Guo et al.~\cite{guo2017} demonstrated that modern neural networks are frequently miscalibrated and popularized temperature scaling. Naeini et al.~\cite{naeini2015} formalized the expected calibration error (ECE) as a standard diagnostic and \cite{tax2026mcgrad} provides open source software for multicalibration. \texttt{proxymate} uses ECE and related calibration diagnostics at the \textit{Unit Level}, and maps calibration failures to appropriate adjustment methods (Platt scaling, isotonic regression, or temperature scaling depending on model complexity and data availability). 

\paragraph{\textit{Estimate}-level adjustment for proxy inference.} Wilkins-Reeves et al.~\cite{wilkinsreeves2026} introduce an estimate-level framework that models residual proxy bias as a random effect across domains. This adjustment inflates proxy-based confidence intervals to achieve calibrated coverage under random distribution shifts, using only aggregate summaries (no individual-level data required). Crucially, the adjustment is modular: it can be layered on top of any proxy adjustment method (including PPI), accounting for residual biases not addressed by those adjustments. \texttt{proxymate} integrates this estimate-level adjustment as an adjustment strategy for cross-domain transportability and for within-domain aggregate inference when historical domains are available. 

\paragraph{Proxy selection and learning.} Zito et al.~\cite{zito2025} propose Pareto-optimal proxy metrics that balance multiple criteria. Tripuraneni et al.~\cite{tripuraneni2024} formalize proxy selection for experimentation, providing criteria for when a proxy is ``good enough'' to support decision-making. Bibaut et al.~\cite{bibaut2024} study learning the covariance of treatment effects across many weak experiments. \texttt{proxymate} is complementary: it validates a given proxy rather than selecting among candidates, though its multi-level diagnostics can inform selection decisions and incorporate Tripuraneni's criteria. 

\paragraph{Covariate shift and representativity.} The covariate shift literature~\cite{shimodaira2000,sugiyama2007,sugiyama2012} addresses distribution mismatch between training and deployment populations via density ratio estimation and reweighting. Rosenbaum~\cite{rosenbaum2002} and Austin~\cite{austin2011} provide foundational treatments of propensity score methods and standardized mean differences (SMD) for assessing covariate balance. \texttt{proxymate} operationalizes these ideas by integrating representativity diagnostics directly into the validation pipeline: before assessing proxy quality, we first verify that the labeled sample on which quality is evaluated is representative of the target population, using SMD and density ratio diagnostics to flag and optionally correct for selection bias. 

{\sloppy
\paragraph{Prevalence estimation with classifiers.} A body of work addresses population-level inference from classifier outputs. Rogan and Gladen~\cite{rogan1978} provided early methods for prevalence estimation using sensitivity and specificity. Forman~\cite{forman2008} developed quantification learning methods that estimate class proportions without requiring calibrated individual predictions. Wu and Resnick~\cite{wu2024} propose calibrate\-/extrapolate methods for prevalence estimation with black-box classifiers. Nguyen et al.~\cite{nguyen2020clara} develop the CLARA framework for confidence of labels and raters for noisy label problems. These methods address specific instances of the proxy validation problem; \texttt{proxymate} provides the diagnostic scaffolding to determine which adjustment is appropriate in a given setting. \par}

\paragraph{Transportability and meta-analysis.} Pearl and Bareinboim~\cite{pearl2016} formalize transportability as a binary verdict based on causal graph structure; \texttt{proxymate} replaces this with continuous diagnostics based on the heterogeneity machinery of DerSimonian and Laird~\cite{dersimonian1986} ($\tau^2$) and Higgins et al.~\cite{higgins2003} ($I^2$). When heterogeneity is detected, \texttt{proxymate} applies the estimate-level adjustment of Wilkins-Reeves et al.~\cite{wilkinsreeves2026} to inflate confidence intervals appropriately. 

\paragraph{Conformal prediction and distribution-free inference.} Conformal prediction~\cite{vovk2005} and conformal risk control~\cite{angelopoulos2023conformal} provide distribution-free coverage guarantees under exchangeability. These methods construct prediction sets with guaranteed marginal coverage, which could complement proxy validation by providing unit-level uncertainty quantification. Measurement invariance testing from psychometrics assesses whether a measurement instrument functions equivalently across groups, conceptually related to our cross-domain transportability checks. \texttt{proxymate} does not currently integrate conformal methods but they represent a natural extension for distribution-free calibration diagnostics. 

\paragraph{Post-prediction inference.} Wang et al.~\cite{wang2020} formalized the ``post-prediction inference'' problem (using ML-predicted outcomes in downstream statistical analyses) and developed adjustment methods based on modeling the relationship between predictions and true outcomes. Chen et al.~\cite{chen2025} provide a unified framework for inference under general missingness patterns with machine learning imputation. Mozer~\cite{mozer2026} connects PPI to classical survey sampling estimators, showing that PPI is a difference estimator. These works address specific adjustment mechanisms; \texttt{proxymate} provides the diagnostic wrapper that determines when and how to apply them. 

\paragraph{Existing software.} Despite the extensive literature, no existing Python package provides an end-to-end proxy validation workflow. Causal inference packages such as DoWhy~\cite{sharma2020} and EconML focus on treatment effect estimation, not proxy validation. The \texttt{ppi\_py} package implements PPI but does not provide diagnostic checks. The \texttt{ppi\_aipw} package~\cite{vanderlaan2026} adds calibration to PPI but does not address representativity or multi-level validation. Experimentation platforms (GrowthBook, Statsig) offer metric monitoring but lack formal proxy validation. \texttt{proxymate} fills this gap by providing a unified diagnose-then-correct workflow spanning four levels of proxy quality.

\end{document}